\documentclass[11pt]{article}
\newif\ifstoc
\stoctrue 
\stocfalse

\usepackage{natbib}
\usepackage{times}
\usepackage{fullpage}
\usepackage{geometry}
\usepackage{amstext,amssymb,amsmath,amsthm}
\usepackage{verbatim}
\usepackage{color}
\usepackage{paralist}
\usepackage{multirow}
\usepackage{array}
\usepackage{graphicx}
\usepackage{color}
\usepackage{xspace}
\usepackage{algorithm}
\usepackage{algorithmic}
\usepackage{times}
\usepackage{dsfont}
\usepackage{caption}
\usepackage{subcaption}

\newcolumntype{x}[1]{>{\centering\arraybackslash}p{#1}}
\newlength\SUBSIZE

\newcommand{\X}{\mathcal{X}}
\newcounter{cnt}
\newtheorem{lem}[cnt]{Lemma}

\newtheorem{thm}[cnt]{Theorem}

\newtheorem{defn}[cnt]{Definition}

\newtheorem{corollary}[cnt]{Corollary}
\newtheorem{assumption}[cnt]{Assumption}

\newcommand{\minCo}{{\hat j}}
\newcommand{\ld}[1]{\|#1\|_\D}
\newcommand{\ip}[2]{\langle #1,#2\rangle}
\newcommand{\ipd}[2]{\langle #1,#2\rangle_\D}
\renewcommand{\algorithmicrequire}{\textbf{Input:}}
\renewcommand{\algorithmicensure}{\textbf{Output:}}
\newcommand{\nptheta}{\widehat\theta}

\renewcommand{\paragraph}[1]{\vspace{3pt}\noindent\textbf{#1}}

\newcommand{\ltwo}[1]{\|#1\|_2}

\newcommand{\A}{\mathcal{A}}
\newcommand{\D}{\mathcal{D}}

\newcommand{\I}{\mathbb{I}}
\DeclareMathOperator*{\E}{\mathbb{E}}

\newcommand{\rL}{{\sf Risk}}

\long\def\cut#1{{}}
\newcommand{\re}{\mathbb{R}}
\newcommand{\R}{\mathbb{R}}

\newcommand{\grad}{\bigtriangledown}
\newcommand{\mypar}[1]{\smallskip
\noindent{\bf\em {#1}.}}

\newcommand{\ignore}[1]{}

\newcommand{\hJ}{{\widehat J}}

\newcommand{\C}{\mathcal{C}}

\usepackage[pdftex]{hyperref}
\usepackage{authblk}
\newcommand{\proj}{{\sf \Pi}}

\newcommand{\diag}{{\sf diag    }}

\long\def\cut#1{{}}

\ifstoc
\else
\fi

\begin{document}
%
\title{To Drop or Not to Drop: Robustness, Consistency and\\ Differential Privacy Properties of Dropout}
\author[1]{Prateek Jain}
\author[3]{Vivek Kulkarni}
\author[2]{Abhradeep Thakurta}
\author[4]{Oliver Williams}
\affil[1]{Microsoft Research}
\affil[2]{Yahoo Labs}
\affil[3]{Stony Brook University}
\affil[4]{Apple Inc}

\date{}
\maketitle
\begin{abstract}
Training deep belief networks (DBNs) requires optimizing a non-convex function with an extremely large number of parameters. Naturally, existing gradient descent (GD) based methods are prone to arbitrarily poor local minima. In this paper, we rigorously show that such local minima can be avoided (upto an approximation error) by using the dropout technique, a widely used heuristic in this domain. In particular, we show that by randomly dropping a few nodes of a one-hidden layer neural network, the training objective function, up to a certain approximation error, decreases by a multiplicative factor.

On the flip side, we show that for training {\em convex} empirical risk minimizers (ERM), dropout in fact acts as a ``stabilizer'' or regularizer. That is, a simple dropout based GD method for convex ERMs is stable in the face of arbitrary changes to any one of the training points. Using the above assertion, we show that dropout provides fast rates for generalization error in learning (convex) generalized linear models (GLM). Moreover, using the above mentioned stability properties of dropout, we design dropout based differentially private algorithms for solving ERMs.  The learned GLM thus, preserves privacy of each of the individual training points while providing accurate predictions for new test points. Finally, we empirically validate our stability assertions for dropout in the context of convex ERMs and show that surprisingly, dropout significantly outperforms (in terms of prediction accuracy) the $L_2$ regularization based methods for several benchmark datasets.
\end{abstract}
\thispagestyle{empty}
\setcounter{page}{0}
\newpage
\section{Introduction}
\label{sec:intro}
Recently, deep belief networks (DBNs) have been used to design state-of-the-art systems in several important learning applications. An important reason for the success of DBNs is that they can model complex prediction functions using a large number of parameters linked through non-linear gating functions. However, this also makes training such models an extremely challenging task. Since there are potentially a large number of local minimas in the space of parameters, any standard gradient descent style method is prone to getting stuck in a local minimum which might be arbitrarily far from the global optimum.

A popular heuristic to avoid such local minima is \emph{dropout} which perturbs the objective function randomly by dropping out several nodes of the DBN. Recently, there has been some work to understand this heuristic in certain limited convex settings \cite{baldi2014dropout,wager2013dropout}. However, in general the heuristic is not well understood, especially in the context of DBNs.

In this work, we first seek to understand why and under what conditions dropout helps in training DBNs. To this end, we show that for fairly general one-hidden layer neural networks, dropout indeed helps avoid local minima/stationary points. We prove that the following holds with at least a constant probability: dropout decreases the objective function value by  a multiplicative factor, as long as the objective function value is not close to the optimal value (see Theorem \ref{thm:boundDropComplex}). 
 To the best of our knowledge, ours is the first such result that explains performance of dropout for training neural networks.

Recently in a seminal work, \cite{AndoniPV014} showed rigorously that a gradient descent based method for neural networks can be used to learn low-degree polynomials. However, their method analyzes a complex perturbation to gradient descent and does not apply to dropout. Moreover, our results apply to a significantly more general problem setting than the ones considered in \cite{AndoniPV014}; see Section~\ref{sec:compAndoni} for more details. 

\mypar{Excess Risk bounds for Dropout} Additionally, we also study the dropout heuristic in a relatively easier setting of convex empirical risk minimization (ERM), where gradient descent methods are known to converge to the global optimum. In contrast to the above mentioned ``instability'' result, for convex ERM setting, our result indicates that the dropout heuristic leads to ``stability'' of the optimum. This hints at a dichotomy that dropout makes the global optimum stable while de-stabilizing local optima. 

In particular, we study the excess error incurred by the dropout method when applied to the convex ERM problem. We show that, in expectation, dropout solves a problem similar to weighted $L_2$-regularized ERM and exhibits  fast excess risk rates (see Theorem~\ref{thm:dropoutGenSGD}). 
In comparison to recent works that analyze dropout for ERM style problems \cite{baldi2014dropout,wager2014altitude}, we study the general problem of convex ERM in generalized linear model (GLM) and provide precise generalized error bounds for the same. See Section~\ref{sec:compRelGen} for more details. 

\mypar{Private learning using dropout} Privacy is a looming concern for several large scale machine learning applications that have access to potentially sensitive data (e.g., medical health records) \cite{Dwork06}. Differential privacy \cite{DMNS06} is a cryptographically strong notion of statistical data privacy. It has been extremely effective in protecting privacy in learning applications \cite{CMS11,DuchiJW13,song2013stochastic,jain2014near}.

As mentioned above, for convex ERMs, dropout can be shown to be ``stable'' w.r.t. changing one or few entries in the training data. Using this insight, we design a dropout based differential private algorithm for convex ERMs (in GLM). Our  algorithm requires that, in expectation over the randomness of dropout, the minimum eigenvalue of the Hessian of the given convex function should be lower bounded. This is in stark contrast to the existing differentially private learning algorithms. Most of these methods either need a strongly convex regularization or assume that the given ERM itself is strongly convex.

\mypar{Experimental evaluation of dropout} Finally, we empirically validate our stability and ``regularization'' assertion for dropout in the convex ERM setting.  
In particular, we focus on the stability of dropout w.r.t.~removal of training data, i.e., LOO stability. We study the random and adversarial removal of data samples. Interestingly, a recent works by \cite{SzegedyZSBEGF13} and \cite{maaten2013learning} provide a complementary set of experiments: while we study dropout for adversarial removal of the training data, \cite{SzegedyZSBEGF13} studies adversarial perturbation of test inputs and \cite{maaten2013learning} considers corrupted features. Our experiments indicate that dropout engenders more stability in accuracy than $L_2$ regularization(with appropriate cross-validation to tune the regularization parameter). Moreover, perhaps surprisingly, dropout yields a more accurate classifier than the popular $L_2$ regularization for several datasets. For example, for the Atheist dataset from UCI repository, dropout based logistic regression is almost 3\% more accurate than the $L_2$ regularized logistic regression. 

{\bf Paper Organization}:  We present our analysis of dropout for training neural networks in Section~\ref{sec:nonRobustDrop}. Then, Section~\ref{sec:dropoutGen} presents excess risk bounds for dropout when applied to the convex ERM problem. 
In Section~\ref{sec:stabDropout}, we show that dropout applied to convex ERMs leads to stable solutions that can be used to guarantee differential privacy for the algorithm. Finally, we present our empirical results in Section~\ref{sec:experiments}.

\section{Dropout algorithm for neural networks}
\label{sec:dropoutPrelim}


\label{sec:nonRobustDrop}
In this section, we provide rigorous guarantees for training a certain class of neural networks (which are in particular \emph{non-convex}) using the dropout heuristic. In particular, we show that dropout ensures with a constant probability that gradient descent does not get stuck in a
``local optimum''. In fact under certain assumptions (stated in Theorem \ref{thm:boundDropComplex}), one can show that the function estimation error actually reduces by a \emph{multiplicative factor} due to dropout.  \cite{AndoniPV014} also study the robustness properties of the local optima encountered by the gradient descent procedure while training neural networks. However, their proof applies only for complex perturbation of gradient descent and only for approximating low-degree polynomials.

\mypar{Problem Setting} We first describe the exact problem setting that we study. Let the space of input feature vectors be $\mathcal{X}\subseteq\re^p$. Let $\D$ be a fixed distribution defined on $\X$. For a fixed function $f:\X\to\re$, the {\bf goal} is to approximate the function $f$ with a neural network (which we will define shortly). For a given  estimated function $g:\X\to\re$, the error is measured by $\|g-f\|_\D^2=\E\limits_{x\sim\D}\left[|g(x)-f(x)|^2\right]$. We also define inner-product w.r.t. the distribution $\D$ as $\ipd{g}{x}=\E\limits_{x\sim\D}[g(x)f(x)]$. We now define the architecture of the neural network that is used to approximate $f$.

\mypar{Neural network architecture} We consider a one-hidden layer neural network architecture with $m$ nodes in the hidden layer. Let the underlying function $f$ be given by: $f(x)=\sum_i \alpha_i \phi_i(\ip{\theta_i^*}{x})$, where $\phi_i:\R^p\rightarrow \R$ is the link function for each hidden layer node $i$. For simplicity, we assume that the coefficients $\alpha_i\geq 0$ are fixed  $\forall  1\leq i\leq m$. The goal is to learn parameters $\theta^*_i\in\R^p$ for each node. Also, let $\sum_i\alpha_i=d$. 
 The training data given for the learning task is $\{(x,f(x))\}_{x\sim\D}$. Note that \cite{AndoniPV014} also studies the same architecture but their link functions $\phi_i$ are assumed to be low-degree polynomials. 

\mypar{Dropout heuristic} We now describe the dropout algorithm for this problem. At the $t$-th step (for any $t\geq 1$), sample a data point $(x,f(x))\sim {\cal D}$ and perform gradient descent with learning rate $\eta$: $\theta^t_i=\theta^{t-1}_i-\eta \nabla_{\theta^t_i}(\ell(f, g, \theta^t; x))$, where $\nabla_{\theta^t_i}$ is the gradient of the error in approximation of $f$. That is, $$\ell(f, g, \theta^t; x)=(f(x)-g(x))^2,$$ where $g(x)=\sum_i \alpha_i \phi_i(\ip{\theta^t_i}{x})$ is the $t$-th step approximation to $f$.

Now, if the procedure is stuck in a local minimum, then we use dropout perturbation to push it out of the local minima. That is, select a vector $\{b_i\}_{i\in[m]}$, where each $b_{i}\sim_{\sf unif}\{0,1\}$. Now, for the current estimation $g$ (at time step $t$) we obtain a new polynomial $\hat g$ as: $$\hat g(x)=2\sum\limits_{i\in[m]}\alpha_i b_i g_i (x),$$ where $g_i(x)=\phi_i(\theta^t_i(x))$. We now perform the gradient descent procedure using this perturbed $\hat g$ instead of the true $t$-th step iterate $g$.


We now analyze the effectiveness of the above dropout heuristic for function approximation. We would like to stress  that the objective in this section is to demonstrate instability of local-minima in function approximation  w.r.t. dropout perturbation. This entails that if the gradient descent procedure is stuck in a local minima/stationary point, then dropout heuristic helps get out of local minima and in fact reduces the estimation error significantly.  
However, the exposition here does not guarantee that the dropout algorithm reaches the global minimum. It just ensures that using dropout one can get out of the local minimum.

\begin{thm}
	Let $f=\sum_i \alpha_i \phi_i(\ip{\theta^*_i}{x}):\X\to\re$ be the true polynomial for $\X\subseteq\re^p$, where $\phi_i$ represents the $i$-th node's link function in the neural network. Let $\langle\alpha_1,\cdots,\alpha_m\rangle$ represent the weights on the output layer of the neural network, with $\alpha_{\min}=\min\limits_{i\in[m]}|\alpha_i|$.  Let $g=\sum_i \alpha_i \phi_i(\ip{\theta^t_i}{x})=\sum_i \alpha_i g_i:\X\to\re$ be the current estimate of $f$.

	Let $\D$ be a fixed distribution on $\X$ from which the training examples $(x,f(x))$ are drawn. If $\ld{g}\geq\ld{f}$ and $\ld{g-f}^2\geq\frac{\ltwo{\alpha}^2\max\limits_{i\in [m]}\ld{g_i}^2\sqrt m}{16\sqrt{\alpha_{\min}}}$, then with probability at least $1/8$ over the dropout, the dropped out neural network $\hat g$ satisfies the following.
	$$\ld{\hat g-f}^2\leq\left(1-\sqrt\frac{\alpha_{\min}}{16 m}\right)\ld{g-f}^2$$
	Also, $\E\limits_{b_i,i\in[m]}\left[\hat g(x)\right]=g(x)$ for all $x\in\X$.
\label{thm:boundDropComplex}
\end{thm}

At a high level, the above theorem shows that if the estimation error $\ld{g-f}^2$ is large enough, then the following holds with at least a constant probability: $\hat g$, which is a dropout based perturbation of $g$, has significantly lesser estimation error than $g$ itself. Next in Section \ref{sec:compAndoni} we apply our results to the problem of learning low-degree polynomials and compare our guarantees with those of \cite{AndoniPV014}.

\begin{proof}[Proof of Theorem \ref{thm:boundDropComplex}]
	Let $e=g-f$ be the error polynomial for the approximation $g$ and let $\hat e=\hat g -f$ be the error polynomial for the approximation $\hat g$. We have the following identity. Here $\Delta g=\hat g-g$, and the last step in \eqref{eqn:errorIdent} is for notational purposes.
	\begin{align}
	&\E\limits_{x}\left[|\hat g(x)-f(x)|^2\right]=\E\limits_{x}\left[| g(x)+\Delta g(x)-f(x)|^2\right]\nonumber\\
	&=\E\limits_{x}\left[| g(x)-f(x)|^2\right]+\E\limits_{x}\left[|\Delta g(x)|^2\right]+2\E\limits_{x}\left[\left((g(x)-f(x))\Delta g(x)\right)\right]\nonumber\\
	&\Leftrightarrow \E\limits_{x}\left[|\hat e(x)|^2\right]-\E\limits_{x}\left[|e(x)|^2\right]=\E\limits_{x}\left[|\Delta g(x)|^2\right]+2\E\limits_{x}\left[\left((g( x)-f( x))\Delta g(x)\right)\right]\nonumber\\
	&\Leftrightarrow \ld{\hat e}^2-\ld{e}^2=\underbrace{\ld{\Delta g}^2}_A+\underbrace{2\ipd{\Delta g}{g-f}}_B
	\label{eqn:errorIdent}
	\end{align}
	We first analyze the term $B$. Notice that one can equivalently write the polynomial $\hat g-g$ as $\sum\limits_{i=1}^m\alpha_i\mu_i g_i$, where $\mu_i\sim_{\sf unif}\{-1,1\}$. We have.
	\begin{equation}
	\E\limits_{\mu}\left[\sum\limits_{i\in[m]}\alpha_i \mu_i \ipd{g_i}{g-f}\right]=0.
	\label{eq:zeroMean}
	\end{equation}
	In order to lower bound $B$ in \eqref{eqn:errorIdent}, along with \eqref{eq:zeroMean} we now need to lower bound the variance of $B$, i.e., lower bound the random variable $Z=\ipd{\Delta g}{g-f}^2$. By the randomness of $\mu_i$'s we have $\E\limits_{\mu}[Z]=\sum\limits_{i=1}^m\alpha_i^2\ipd{g_i}{g-f}^2$. Also by \eqref{eq:lbErrV}, we have $\E\limits_\mu[Z^2]\leq 2\left(\sum_{i}\alpha_i^2\ipd{g_i}{g-f}^2\right)^2$. Using standard Payley-Zigmund anti-concentration inequality, we have $\Pr[Z\geq\E[Z]/2]\geq\frac{(\E[Z]^2)}{4\E[Z^2]}$. Plugging in the bounds on $Z$ from above we have $\Pr[Z\geq\E[Z]/2]\geq 1/4$.
	
	\begin{align}
	\E\limits_{\mu}[Z^2]&=\E\limits_{\mu}\left[\left(\sum\limits_{i=1}^m\alpha_i\mu_i\ipd{g_i}{g-f}\right)^4\right]\nonumber\\
	&=\E\limits_{\mu}\left[\sum\limits_{i,j,k,\ell}(\alpha_i\alpha_j\alpha_k\alpha_\ell)\cdot(\mu_i\mu_j\mu_k\mu_\ell)\ipd{g_i}{g-f}\ipd{g_j}{g-f}\ipd{g_k}{g-f}\ipd{g_\ell}{g-f}\right].\nonumber\\
	&\leq 2\E\limits_{\mu}\left[\left(\sum_{i,j}\alpha_i\alpha_j\mu_i\mu_j\ipd{g_i}{g-f}\ipd{g_j}{g-f}\right)^2\right]\nonumber\\
	&=2\sum_{i,j}\alpha_i^2\alpha_j^2\ipd{g_i}{g-f}^2\ipd{g_j}{g-f}^2\nonumber\\
	&\leq 2\left(\sum_{i}\alpha_i^2\ipd{g_i}{g-f}^2\right)^2.
	\label{eq:lbErrV}
	\end{align}
	
	In \eqref{eq:lb1s} we lower bound $\E\limits_{\mu}[Z]$ as follows. In \eqref{eq:lb1s} we have used the assumption that $\ld{g}^2\geq\ld{f}^2$. Using the fact in \eqref{eq:zeroMean} and \eqref{eq:lb1s}, it follows that with probability at least $1/4$, we have $2\ipd{\Delta g}{g-f}\geq -\frac{\alpha_{\min}}{2\sqrt m}\ld{g-f}^2$.
	\begin{align}
	\E\limits_{\mu}[Z]&=\sum\limits_{i=1}^m\alpha_i\ipd{\alpha_i g_i}{g-f}^2\\
    &\geq\frac{\alpha_{\min}}{m}\left(\sum\limits_{i=1} ^m\left|\ipd{\alpha_i g_i}{g-f}\right|\right)^2\nonumber\\
	&\geq\frac{\alpha_{\min}}{m}\ipd{g}{g-f}^2=\frac{\alpha_{\min}}{4m}\ld{g-f}^4
	\label{eq:lb1s}
	\end{align}
	Now we focus on the term $A$ in \eqref{eqn:errorIdent} and provide an upper bound. We have $\E\limits_{\mu}\left[\ld{\Delta g}^2\right]=\sum\limits_{i=1}^m\alpha_i^2\ld{g_i}^2$. Using the bounds on the terms $A$ and $B$ one can conclude that if $\ld{g-f}^2\geq\frac{\ltwo{\alpha}^2\max\limits_{i\in [m]}\ld{g_i}^2\sqrt m}{16\sqrt{\alpha_{\min}}}$, then with probability at least $1/8$, $\ld{\hat g-f}^2\leq\left(1-\frac{\sqrt{\alpha_{\min}}}{4\sqrt m}\right)\ld{g-f}^2$. This completes the proof.
\end{proof}

\subsection{Application: Learning polynomials with neural networks}
\label{sec:compAndoni}

The work of \cite{AndoniPV014} studied the problem of learning degree-$d$ polynomials (with real or complex coefficients) using polynomial neural networks described above. In this section we provide a comparative analysis of \citet[Theorem 5.1]{AndoniPV014} with Theorem \ref{thm:boundDropComplex} above. The approach of \cite{AndoniPV014} is different from our approach in two ways: i) For the analysis of \cite{AndoniPV014} to go through, the perturbation has to be complex, and ii) They consider additive perturbation to the weights as opposed to the multiplicative perturbation to the nodes exhibited by dropout.

	In order to make the results comparable, we will assume that for each of the node $i$, $\ld{g_i}=\Theta(1)$, and $\alpha_i=\Theta(1)$. (Since \cite{AndoniPV014} deal with complex numbers, these bounds above are on the modulus.) Under this assumption, Theorem \ref{thm:boundDropComplex} suggests that the error can be brought down to $O(m\sqrt m)$ where as \citet[Theorem 5.1]{AndoniPV014} show that the error can be brought down to $O(mp^d)$. Notice that our bound is independent of the dimensionality ($p$) and the degree of the polynomial ($d$). In terms of the rate of convergence, \citet[Theorem 5.1]{AndoniPV014} ensures that the error reduces by $\Omega(1-1/m^{d})$ factor, while in our case it is $\Omega(1-1/\sqrt m)$. Another advantage of Theorem \ref{thm:boundDropComplex} is that it is oblivious to the data distribution $\D$, as opposed to the results of \cite{AndoniPV014} which explicitly require $\D$ to be either uniform or Gaussian.



\section{Fast rates of convergence for dropout}
\label{sec:dropoutGen}

In the previous section we saw how dropout helps one come out of local minimum encountered during gradient descent. In this section, we show that for generalized linear models (GLMs) (a class of one layer convex neural networks), dropout gradient descent provides an excess risk bound of $O(1/n)$, where $n$ is the number of training data samples.

\mypar{Problem Setting} We first describe the exact problem setting that we study.
Let $\tau(\D)$ be a fixed but unknown distribution over the data domain $\D=\{(x,y):x\in\mathcal{X}, y\in\mathcal{Y}\}$, 
where $\mathcal{X}\subseteq\re^p$ is the input feature domain and $\mathcal{Y}\subseteq\re$ is the target output domain. Let the \emph{loss} $\ell(\theta;x,y)$ be a real-valued convex function (in the first parameter) defined over all $\theta\in\re^p$ and all $(x,y)\in\D$. The population and excess risk of a model $\theta$ are defined as:
\begin{align} \label{eq:risk}
{\sf Risk}(\theta)&=\E\limits_{(x,y)\sim\tau(\D)}\left[\ell(\theta;x,y)\right],\nonumber\\ {\sf ExcessRisk}(\theta)&=\rL(\theta)-\min\limits_{\theta'\in\C}\rL(\theta'),
\end{align}
where $\C\subseteq \re^p$ is a fixed convex set. A learning algorithm $\A$ typically has access to only a set of samples $D=\{(x_1,y_1),\cdots,(x_n,y_n)\}$, drawn i.i.d. from $\tau(\D)$. The goal of the algorithm is to find $\widehat{\theta}$ with small excess risk. 

\mypar{Dropout Heuristic} We now describe the dropout based algorithm used to minimize the Excess Risk (see \eqref{eq:risk}). At a high-level, we just use the standard stochastic gradient descent algorithm. However, at each step,a random $\alpha$-fraction of the coordinates of the parameter vector are updated. That is, the data point $x$ generated by the stochastic gradient descent is perturbed to obtain $b*x\in \R^p$ where the $i$-th coordinate of $b*x$ is given by $b_i x_i$. Now, the perturbed $b*x$ is used to update the parameter vector $\theta$. In this section, we assume that the sampling probablity $\alpha=1/2$. See Algorithm \ref{Algo:privClass} for the exact dropout algorithm that we analyze.   

We also analyze a stylized variant of dropout that can be effectively captured by a standard regularized empirical risk minimization setup. (See Appendix \ref{sec:ERM}.) Both of these analyses hinge on the observation that even though the loss functions are not strongly convex in general, the dropout variants of these loss functions are \emph{strongly convex in expectation} and enable us to derive an excess risk of $O(1/n)$ in both cases. Recall for non-strongly convex loss functions in general, the lower bound on excess risk is $O(1/\sqrt n)$ \cite{SSSS}.
\begin{algorithm}[tb]
	\caption{Dropout gradient descent}
	\begin{algorithmic}[1]
		\REQUIRE Data set: $D={(x_1,y_1),\cdots,(x_n,y_n)}$, loss function: $\ell$, learning rate: $\eta$, dropout probability: $\alpha$, $T$: Number of iterations of SGD.
		\STATE Choose initial starting model: $\theta_1$.
		\FOR{$t\in[T]$}
			\STATE Sample $(x_t,y_t)\sim_{\sf iid} D$, $b\sim_{\sf unif} \{0,1\}^p$.
			\STATE {$\theta_{t+1}\leftarrow\theta_{t}-\frac{\eta}{\alpha}\grad\ell(\ip{\theta_t}{b*x},y)$.\label{line:abcd}}
		\ENDFOR
		\STATE {\bf Output}: $\theta_{T+1}$.
	\end{algorithmic}
	\label{Algo:privClass}
\end{algorithm}



\begin{assumption}[Data normalization]
	i) For any $(x,y)\in\D$, $\ltwo{x}\leq B$, and ii) The loss function $\ell(u;y)$ is $1$-strongly convex in $u$ (i.e., $\frac{\partial^2 \ell(u; y)}{\partial^2 u}\geq 1$) and $G$-Lipschitz (i.e, $|\frac{\partial \ell(u; y)}{\partial u}|\leq G$).
	\label{assump:normDat}
\end{assumption}

In Theorem \ref{thm:dropoutGenSGD} we provide the excess risk guarantee for the dropout heuristic.

\begin{thm}[Dropout generalization bound]
	Let $\C\subseteq\re^p$ be a fixed convex set and let Assumption \ref{assump:normDat} be true for the data domain $\D$ and the loss $\ell$. Let $D=\{(x_1,y_1),\cdots,(x_n,y_n)\}$ be $n$ i.i.d. samples drawn from $\tau(\D)$. Let $\rL(\theta)$ be defined as in \eqref{eq:risk12}. Let the learning rate $\eta_t=\frac{1}{\Delta_1 t}$. Then over the randomness of the SGD algorithm and the distribution $\tau(\D)$, we have excess risk
	\begin{multline*}\E_{(x,y)\sim\tau(\D), b}\left[\ell(2\ip{x*b}{\theta};y)\right]-\min\limits_{\theta\in\C}\E_{(x,y)\sim\tau(\D), b}\left[\ell(2\ip{x*b}{\theta};y)\right]=O\left(\frac{(GB)^2\log T}{\Delta_1 T}+\frac{(GB)^2\log n}{\Delta n}\right).\end{multline*}
	Here $\Delta=\min\limits_{j\in[p]}{\E\limits_x\left[x(j)^2\right]}$, $\Delta_1=\min\limits_{j\in[p]}\frac{1}{n}\sum\limits_{i=1}^n x_i(p)^2$, and $G$ and $B$ are defined in Assumption \ref{assump:normDat}. The outer expectation is over the randomness of the algorithm.
	\label{thm:dropoutGenSGD}
\end{thm}
The proof of this theorem is provided in Section \ref{app:dropoutGenSGD}. Observe that if  $T\geq n^2$, and $\Delta,\Delta_1$ are assumed to be constants, then the excess risk bound of Theorem~\ref{thm:dropoutGenSGD} is $O(1/n)$. Second note that the bound is for the dropout risk defined in \eqref{eq:risk12}. For the special case of linear regression (see Lemma \ref{lemma:linReg}), dropout-based risk is the true risk \eqref{eq:risk} plus $L_2$ regularization. Hence, in this case, using standard arguments \cite{SSSS} we get the $1/\sqrt{n}$ excess risk rate for {\em population} risk defined in \eqref{eq:risk}. However, for other loss functions, it is not clear how close the dropout based risk is to the population risk.

\begin{lem}
	Let $b$ be drawn uniformly from $\{0,1\}^p$ and let $\ell(2\ip{x*b}{\theta};y)$ be the least squares loss function, i.e., $\ell(2\ip{x}{\theta};y)=(2\ip{x*b}{\theta}-y)^2$. Then,
	\begin{equation}
	\textsf{Risk}(\theta)=\E\limits_{(x,y)\sim\tau(\D)}\left[(y-\ip{x}{\theta})^2\right]+\theta^T \E\limits_{(x,y)\sim\tau(\D)}\left[\diag(xx^T)\right]\theta.
	\end{equation}
	\label{lemma:linReg}
\end{lem}

 See Appendix \ref{app:dropoutGenSGD} for a detailed proof of Theorem \ref{thm:dropoutGenSGD} and Lemma \ref{lemma:linReg}.

 \mypar{Note} Notice that even when $\sum\limits_{i=1}^n x_ix_i^T$  is not full rank (e.g., all the $x_i$ are scaled versions of the $p$-dimensional vector $\langle 1,1,\cdots\rangle$), we can still obtain an excess risk of $O(1/n)$ for the dropout loss. Recall that in general for non-strongly convex loss functions, the best excess risk one can hope for is $O(1/\sqrt n)$ \cite{SSSS}.

\subsection{Comparison to related work}
\label{sec:compRelGen}

After the seminal paper of \cite{hinton2012improving}, demonstrating strong experimental advantage of ``dropping off'' of nodes in the training of deep neural networks, there have been a series of works providing strong theoretical understanding of the dropout heuristic \cite{baldi2014dropout,wager2013dropout,wang2013feature,van2014follow,HelmboldL14,wager2014altitude,mcallester2013pac,maaten2013learning}. A high-level conclusion from all these works has been that dropout behaves as a regularizer, and in particular as an $L_2$ regularizer when the underlying optimization problem is convex. In terms of rates of convergence, the work of \cite{wager2013dropout} provide asymptotic consistency for the  dropout heuristic w.r.t. convex models. They show (using second order Taylor approximation) that asymptotically dropout behaves as an adaptive $L_2$-regularizer. The work of \cite{wager2014altitude} provide the precise rate of convergence of the excess risk when the data is assumed to be coming from a Possion generaive model, and the underlying optimization task is topic modeling. For the classic problem of linear optimization over a polytope, dropout recovers essentially the same bound as \emph{follow the perturbed leader} \cite{KV05} while bypassing the issue of tuning the regularization parameter.

In this work we extend this line of work further by providing the precise (non-asymptotic) rate of convergence of the dropout heuristic for arbitrary generalized linear models (GLMs). In essence, by providing this analysis, we close the fourth open problem raised in the work of \cite{van2014follow} which posed the problem of determining the generalization error bound for GLMs. One surprising aspect of our result is that the rate of convergence is $O(1/n)$ (as opposed to $O(1/\sqrt n)$), even when the underlying data covariance matrix $\sum\limits_{i}x_ix_i^T$ is not full-rank.



\section{Private convex optimization using dropout}\label{sec:stabDropout}

In this section we show that dropout can be used to design differentially private convex optimization algorithms. In the last few years, design of differentially private optimization (learning) algorithms have received significant attention \cite{CM08,DL09,CMS11,JKT12,KST12,DuchiJW13,song2013stochastic,jain2014near,BST}. We further extend this line of research to show that dropout allows one to exploit properties of the data (e.g., minimum entry in the diagonal of the Hessian)  to ensure robustness, and hence differential privacy. Differential privacy is a cryptographically strong notion which by now is the de-facto standard for statistical data privacy. It ensures the privacy of individual entries in the data set even in the presence of arbitrary auxiliary information \cite{Dwork06,Dwork08}.

\begin{defn}[$(\epsilon,\delta)$-differential privacy \cite{DMNS06,ODO}]\label{def:diffPrivacy}
	For any pairs of  neighboring data sets $D,D'\in\D^n$ differing in exactly one entry, an algorithm $\A$ is $(\epsilon,\delta)$-differentially private if for all measurable sets $S$ in the range space of $\A$ the following holds:
	\begin{equation}
		\Pr[\A(D)\in S]\leq e^\epsilon\Pr[\A(D')\in S]+\delta.
		\label{eq:diffPriv}
	\end{equation}
	Here, think of $\delta=1/n^{\omega(1)}$ and $\epsilon$ to be a small constant.
\end{defn}

\mypar{Background} At an intuitive level differential privacy ensures that the measure induced on the space of possible outputs by a randomized algorithm $\A$ does not depend ``too much'' on the presence or absence of one data entry. This intuition has two immediate consequences: i) If the underlying training data contains potentially sensitive information (e.g., medical records), then it ensures that an adversary learns almost the same information about an individual independent of his/her presence or absence in the data set, and hence protecting his/her privacy, ii) Since the output does not depend ``too much'' on any one data entry, the Algorithm $\A$ cannot over-fit and hence will provably have good generalization error. Formalizations of both these implications can be found in \cite{Dwork06} and \citet[Appendix F]{BST}. This property of a learning algorithm to not over-fit the training data (also known as \emph{stability}) is known to be both \emph{necessary} and \emph{sufficient} for a learning algorithm to generalize \cite{SSSS,shalev2010learnability,PogVoiRos11,BousquetElisseeff-2002}.

In the following we provide a \emph{stylized} example where dropout ensures differential privacy. In Appendix \ref{app:dropGLM}, we provide a detailed approach of extending this example to arbitrary generalized linear models (GLMs). (See Section \ref{sec:dropoutGen} for a refresher on GLMs.)

\subsection{Private dropout learning over the simplex}\label{sec:linLoss}
In this section, we analyze a \emph{stylized} example:  Linear loss functions over the simplex. The idea is to first show that for a given data set $D$ (with a set fixed set of properties which we will describe in Theorem \ref{thm:dropoutStab}), the dropout algorithm satisfies the \emph{differential privacy condition} in \eqref{eq:diffPriv} for any data set $D'$ differing in one entry from $D$. (For the purposes of brevity, we will refer to \emph{local differential privacy} at $D$.) Later we will use a standard technique called \emph{propose-test-release} (PTR) framework \cite{DL09} to convert the above into a differential privacy guarantee. (The details are given in Algorithm \ref{Algo:privClassSimp}).

\begin{algorithm}[tb]
	\caption{Private dropout learning over the simplex}
	\begin{algorithmic}[1]
		\REQUIRE Data set: $D=\{x_1,\cdots,x_n\}\subset\{0,1\}^p$, $\C=p$-dimensional simplex, privacy parameters $\epsilon$ and $\delta$.
		\STATE $\forall j\in[p], c(j)\leftarrow\min\limits_{k\in[n]}\frac{1}{n}\sum\limits_{i=1,i\neq k}^nx_i(j)$ and $\Lambda=\min\limits_{j\in[p]} c(j)$.
		\STATE $\hat{\Lambda}\leftarrow\Lambda+{\sf Lap}(1/n\epsilon)$ and $\forall i\in[n],  b_i\sim_{\sf unif}\{0,1\}^p$.
		\STATE {\bf If} $\hat{\Lambda}>\frac{2\log(1/\delta)}{n\epsilon}$, {\bf then} output $\nptheta=\arg\min\limits_{\theta\in\C}\sum\limits_{i=1}^n\ip{x_i*b_i}{\theta}$, {\bf else} output {\sf failure}.
	\end{algorithmic}
	\label{Algo:privClassSimp}
\end{algorithm}

Let the data domain $\mathcal{X}$ be $\{0,1\}^p$ and let $X = \{x_1 \dots, x_n\}$ be i.i.d.~samples from a distribution $\tau(\mathcal{X})$ over $\mathcal{X}$. Let the loss function $\ell$ be $\ell(\ip{\theta}{x})=\ip{\theta}{x}$ and let the constraint set $\C$ be the $p$-dimensional simplex. Let $b_1,\cdots,b_n$ be i.i.d. uniform samples from $\{0,1\}^p$. The dropout optimization problem for the linear case can be defined as below. Here the $i$-th coordinate of $b*x$ is given by $b_ix_i$.
\begin{equation}\label{eq:dropLin}
	\nptheta=\arg\min\limits_{\theta\in\C}\sum\limits_{i=1}^n\ip{x_i*b_i}{\theta}
\end{equation}

At a high level Theorem \ref{thm:dropoutStab} states that changing any one data entry in the training data set $D$, changes the induced probability measure on the set of possible outputs by a factor of $\epsilon$, and with a additive slack that is exponentially small in the number of data samples ($n$).

\begin{thm} [Local differential privacy]
	Let $c(j)=\min\limits_{k\in[n]}\frac{1}{n}\sum\limits_{i=1, i\neq k}^n x_i(j)$ and $\Lambda=\min\limits_{j\in[p]} c(j)$. For the given data set $X$ if $\Lambda>1/(n\epsilon)$ and $\epsilon<1$, then the solution $\nptheta$ of dropout ERM \eqref{eq:dropLin} is $(\epsilon,\delta)$-local differentially private at $X$, where $\delta=p\exp\left(-\Omega(\epsilon^2\Lambda n)\right)$.
	\label{thm:dropoutStab}
\end{thm}

\begin{proof}
	First notice that since we are optimizing a linear function over the simplex, the minimizer $\nptheta$ is essentially one of the coordinates in $[p]$. Therefore one can equivalently write the optimization problem in \eqref{eq:dropLin} as follows. Here $x(j)$ refers to the $j$-th coordinate for the vector $x$.
	
	\begin{equation}
		\minCo=\arg\min\limits_{j\in[p]}\sum\limits_{i=1}^n x_i(j)b_i(j)
		\label{eq:dropLinA}
	\end{equation}
	W.l.o.g. we assume that the neighboring data set $\D'$ differs from $\D$ in $x_n$. Also, let $f(j,D)=\sum\limits_{i=1}^n x_i(j)b_i(j)$. Clearly for any $j$, $|f(j,D)-f(j,D')|\leq 1$. In the following we show that the measures induced on the random variable $\hat j$ by \eqref{eq:dropLinA} for data sets $D$ and $D'$ have multiplicative closeness. The analysis of this part closely relates to the differential privacy guarantee under binomial distribution from \cite{ODO}.
	
	For a given $j\in[p]$, let $\nu_j$ be the number of non-zeroes in the $j$-th coordinate of all the $x_i$'s, excluding $x_n$. Therefore, we have the following for any $k<\frac{\nu_j}{2}$.
	
	\begin{equation}
		\frac{\Pr[f(j,D)=\left(\frac{\nu_j}{2}+k\right)]}{\Pr[f(j,D')=\left(\frac{\nu_j}{2}+k\right)]}\leq \frac{\frac{\nu_j}{2}+k+1}{\frac{\nu_j}{2}-k}
		\label{eq:probMeasure}
	\end{equation}
	
	So, as long as $k<\frac{\epsilon\nu_j}{8}$, the ratio in \eqref{eq:probMeasure} is upper bounded by $(1+\epsilon)\leq e^\epsilon$ for $\epsilon<1$. By Chernoff bound, such an event happens with probability at least $\exp(-\Omega(\epsilon^2\nu_j))$. For the lower tail of the binomial distribution, an analogous argument provides such a bound.
	
	One can use the same argument for other coordinates too. Notice $\Lambda=\frac{1}{n}\min\limits_{j\in[p]}\nu_j$. By union bound, for any given $j\in[p]$, the ratio of the measures induced by $f(j,D)$ and $f(j,D')$ on any $k\in\mathbb{Z}$ which has probability measure at least $1-p\exp\left(-\Omega(\epsilon^2\Lambda n)\right)$, is in $[e^{-\epsilon},e^\epsilon]$.
	
	In the following we notice that not only individually each of the coordinates satisfy the multiplicative closeness in measure, in fact $\arg\min\limits_{j\in[p]}f(j,D)$ and $\arg\min\limits_{j\in[p]}f(j,D')$ satisfy analogous closeness in measure, where the closeness is within $[e^{-2\epsilon},e^{2\epsilon}]$. This property follows by using \citet[Theorem 5]{BLST}. This concludes the proof.
\end{proof}

\subsubsection{From local differential privacy to differential privacy}\label{sec:convDiffPrivate}

Notice that Theorem \ref{thm:dropoutStab} (ensuring local differential privacy) is independent of the data distribution $\tau(D)$. This has direct implications for differential privacy. We show that using the \emph{propose-test-release} (PTR) framework \cite{DL09,STLasso13}, the dropout heuristic provides a \emph{differentially private} algorithm.

\mypar{Propose-test-release framework} Notice that for any pair of data sets $D$ and $D'$ differing in one entry, $\Lambda(D)$ and $\Lambda(D')$ in Theorem \ref{thm:dropoutStab} differs by at most $1/n$. So using the standard \emph{Laplace mechanism} from differential privacy \cite{DMNS06}, one can show that $\hat{\Lambda}=\Lambda(D)+{\sf Lap}(1/n\epsilon)$ satisfies $(\epsilon,0)$-differential privacy, where ${\sf Lap}(\lambda)$ is random variable sampled from the Laplace distribution with the scaling parameter of $\lambda$. With $\hat{\Lambda}$ in hand we check if $\hat\Lambda>\frac{2\log(1/\delta)}{\epsilon n}$. For the condition being true, we output $\hat{\theta}$ from \eqref{eq:dropLin} and output a $\bot$ otherwise. Theorem \ref{thm:privLinDrop} ensures that the above PTR framework is $(\epsilon,\delta)$-differentially private.

\begin{thm}
Propose-test-release framework along with dropout (i.e., Algorithm \ref{Algo:privClassSimp}) is $(2\epsilon,\delta)$-differentially private for optimizing linear functions over the simplex, where $\epsilon<1$ and $\delta=1/n^{\omega(1)}$.
\label{thm:privLinDrop}
\end{thm}

This theorem is a direct consequence of Theorem \ref{thm:dropoutStab} and \cite{Thakurta15}. Using the tail property of Laplace distribution, one can show that as long as $\Lambda$ in Theorem \ref{thm:dropoutStab} is at least $\frac{4\log(1/\delta)}{\epsilon n}$, w.p. at least $1-\delta$ the above PTR framework outputs $\hat{\theta}$ from \eqref{eq:dropLin} exactly. While the current exposition of the PTR framework is tuned to the problem of optimizing linear functions over the simplex, a much more general treatment is provided in Appendix \ref{app:dropGLM}.


\section{Experiments}\label{sec:experiments}
\begin{figure*}[tb]
\centering
\begin{tabular}[h]{cccc}
    \hspace*{-15pt}\includegraphics[width=0.28\textwidth]{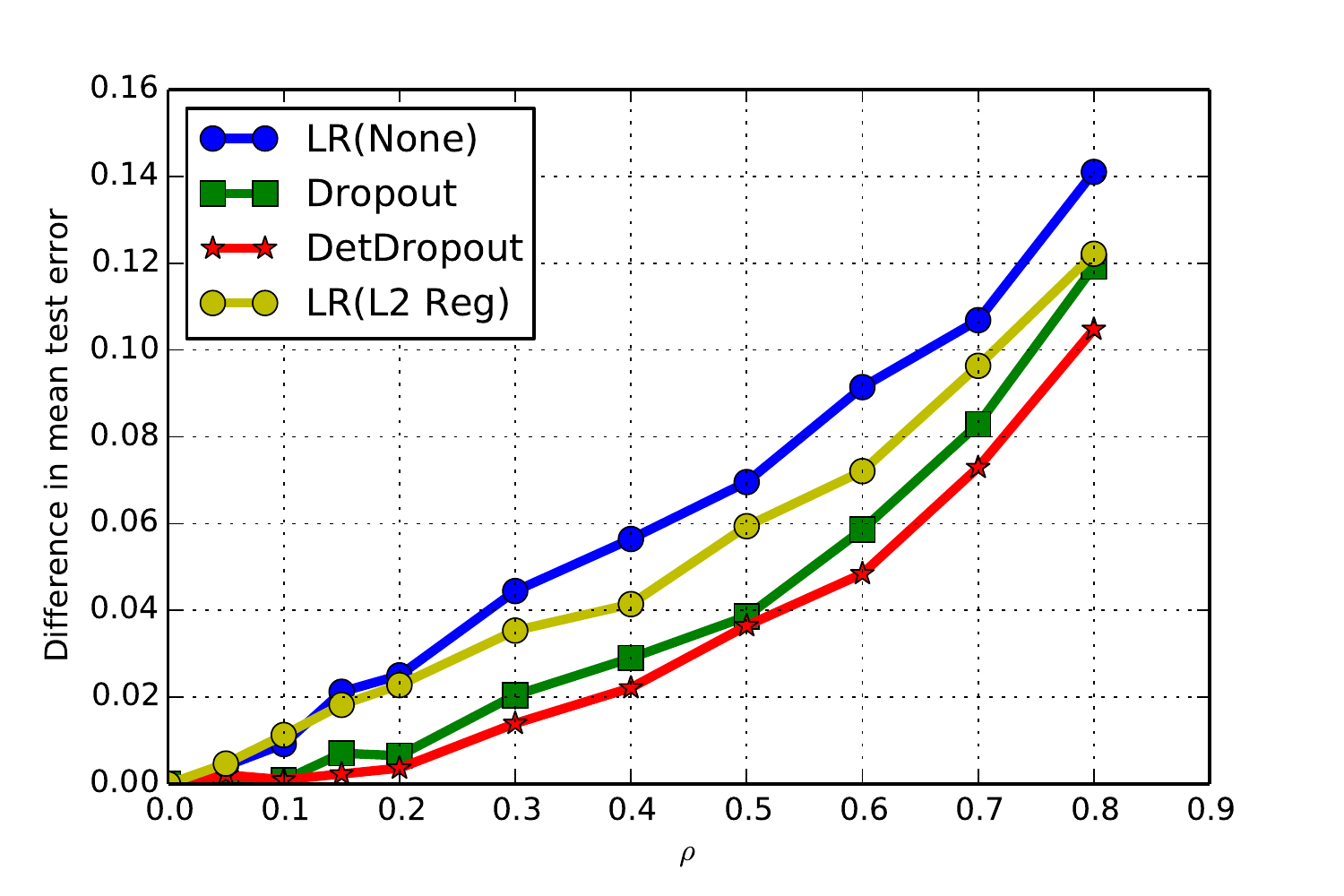}&
    \hspace*{-25pt}\includegraphics[width=0.28\textwidth]{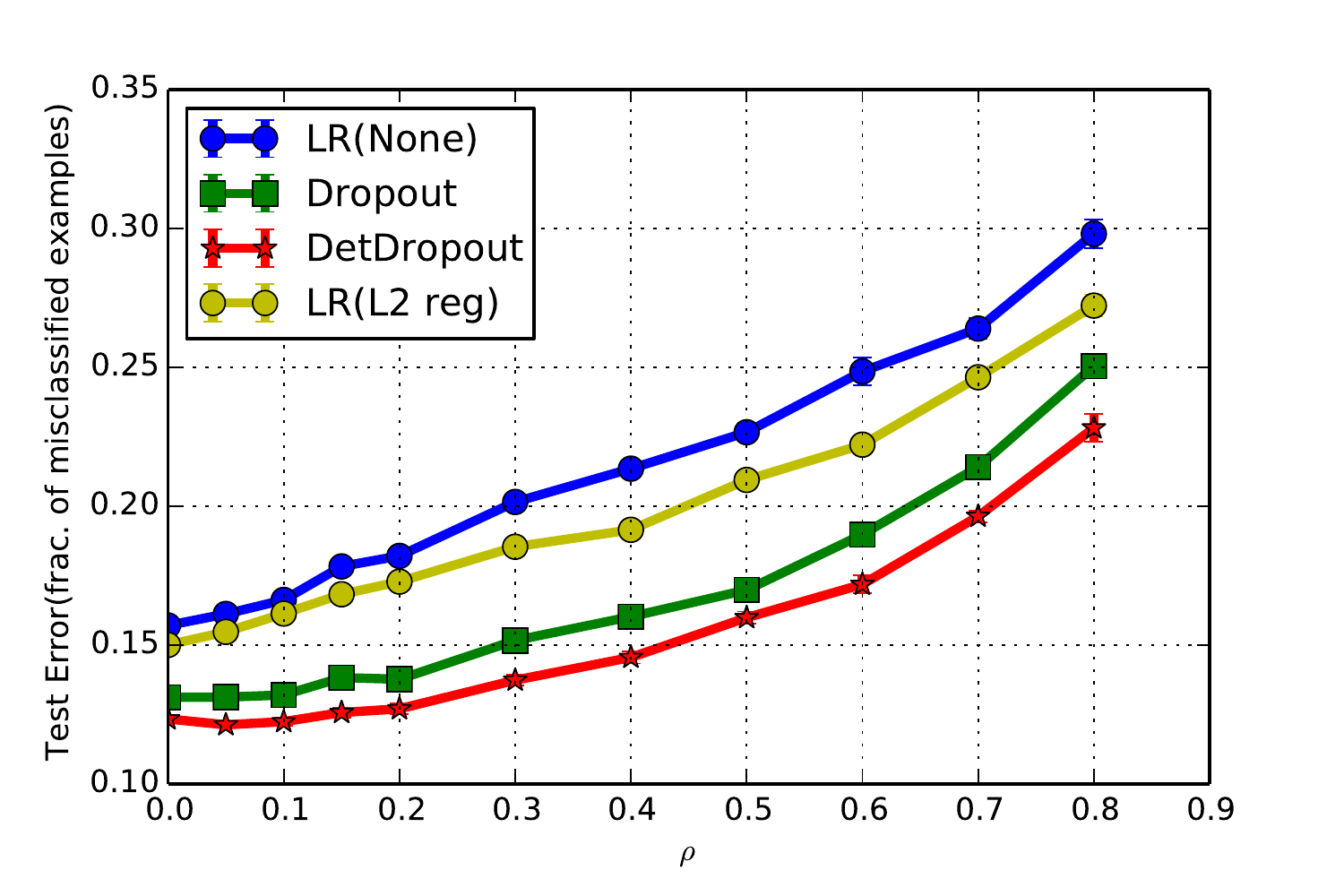}&
        \hspace*{-25pt}\includegraphics[width=0.28\textwidth]{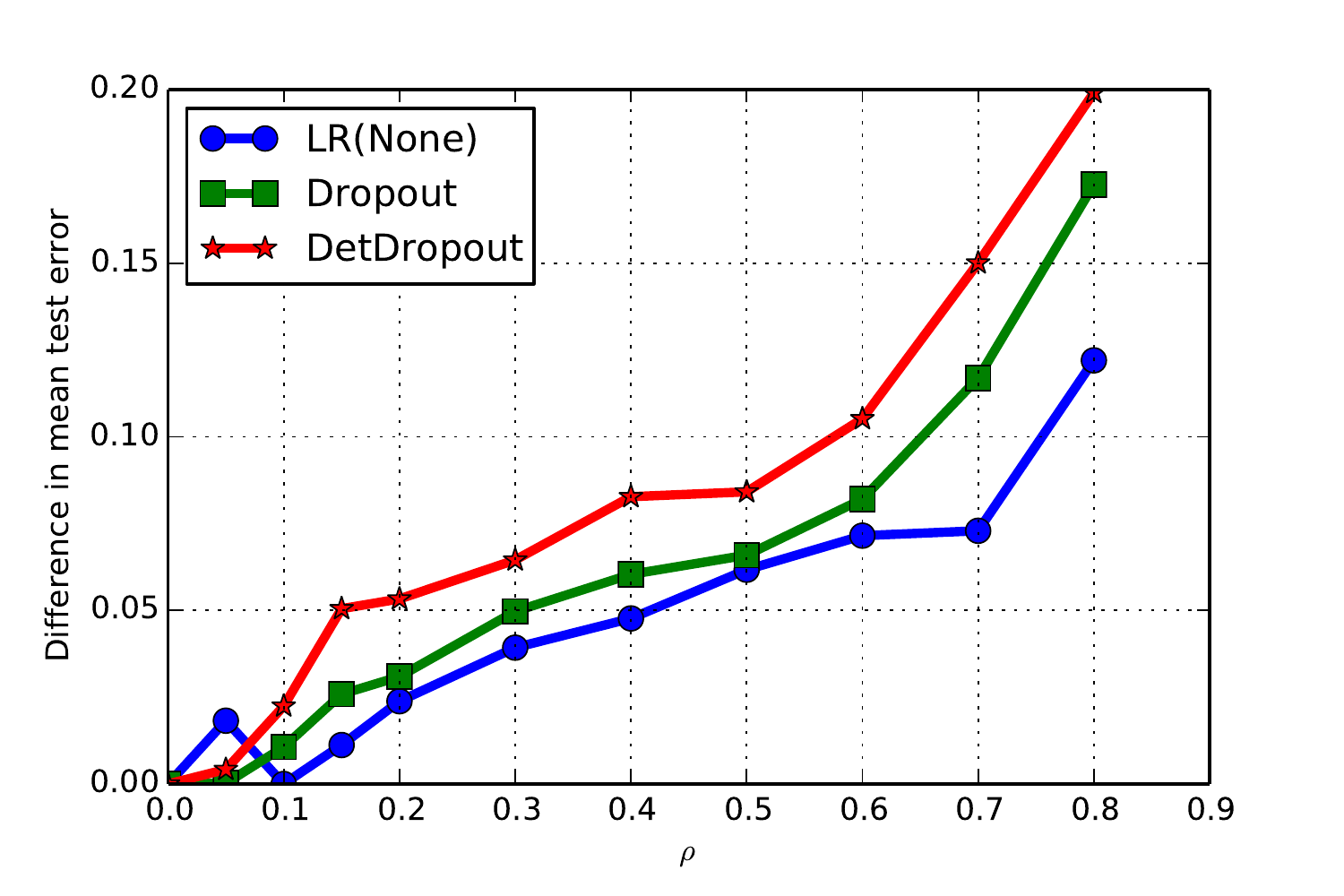}&
    \hspace*{-25pt}\includegraphics[width=0.28\textwidth]{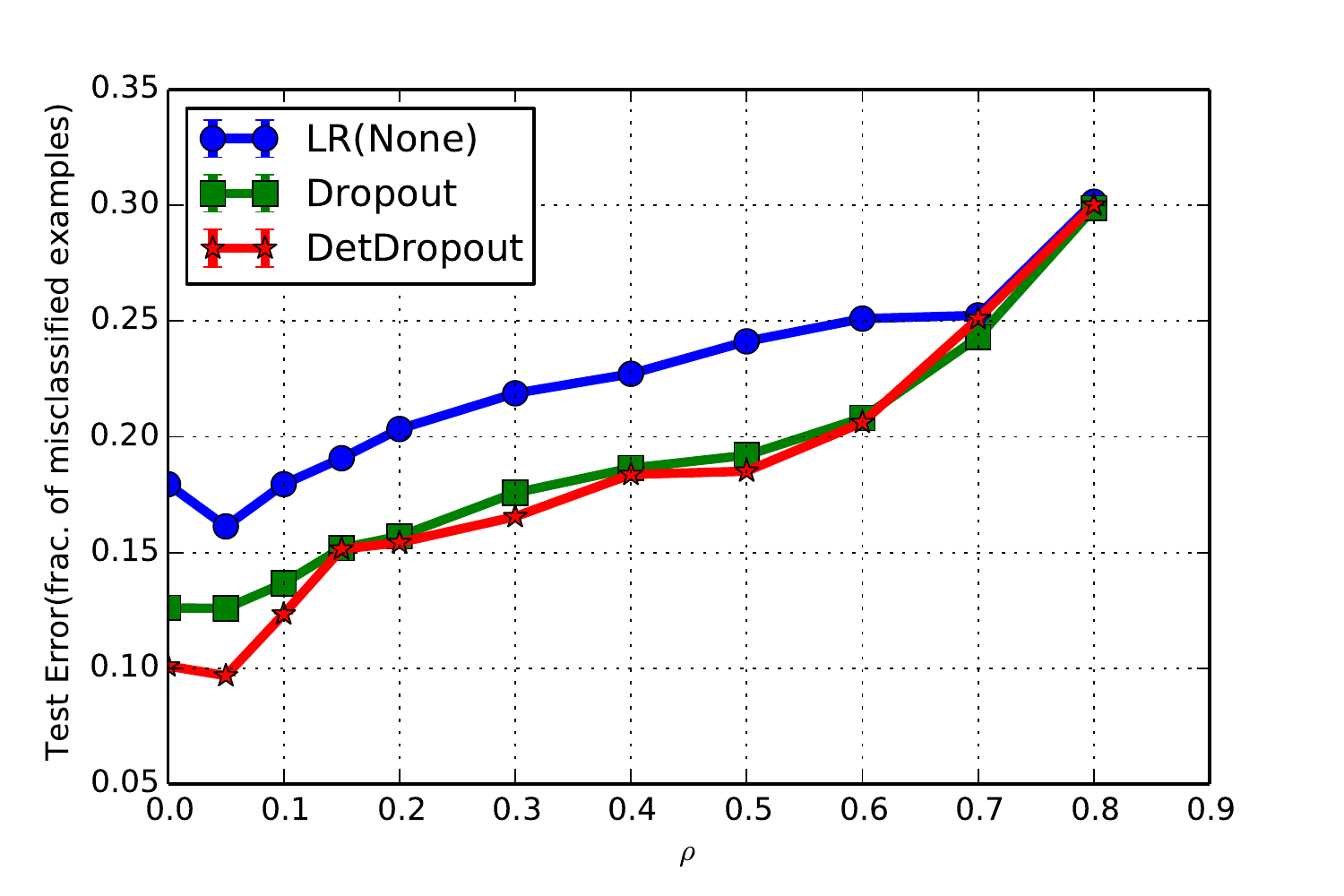}\\
        (a)&(b)&(c)&(d)
\end{tabular}
\caption{Stability analysis for logistic regression on Atheist data set. The experiments were repeated 20 times, the means of which are plotted. (a) and (b) show, for random removal of training examples, how the error, and marginal error vary with $\rho$. (c) and (d) shows the same for adversarial removal.}
\label{fig:stability}
\end{figure*}
\begin{figure*}[tb]
\centering
\begin{tabular}[h]{cccc}
    \hspace*{-15pt}\includegraphics[width=0.28\textwidth]{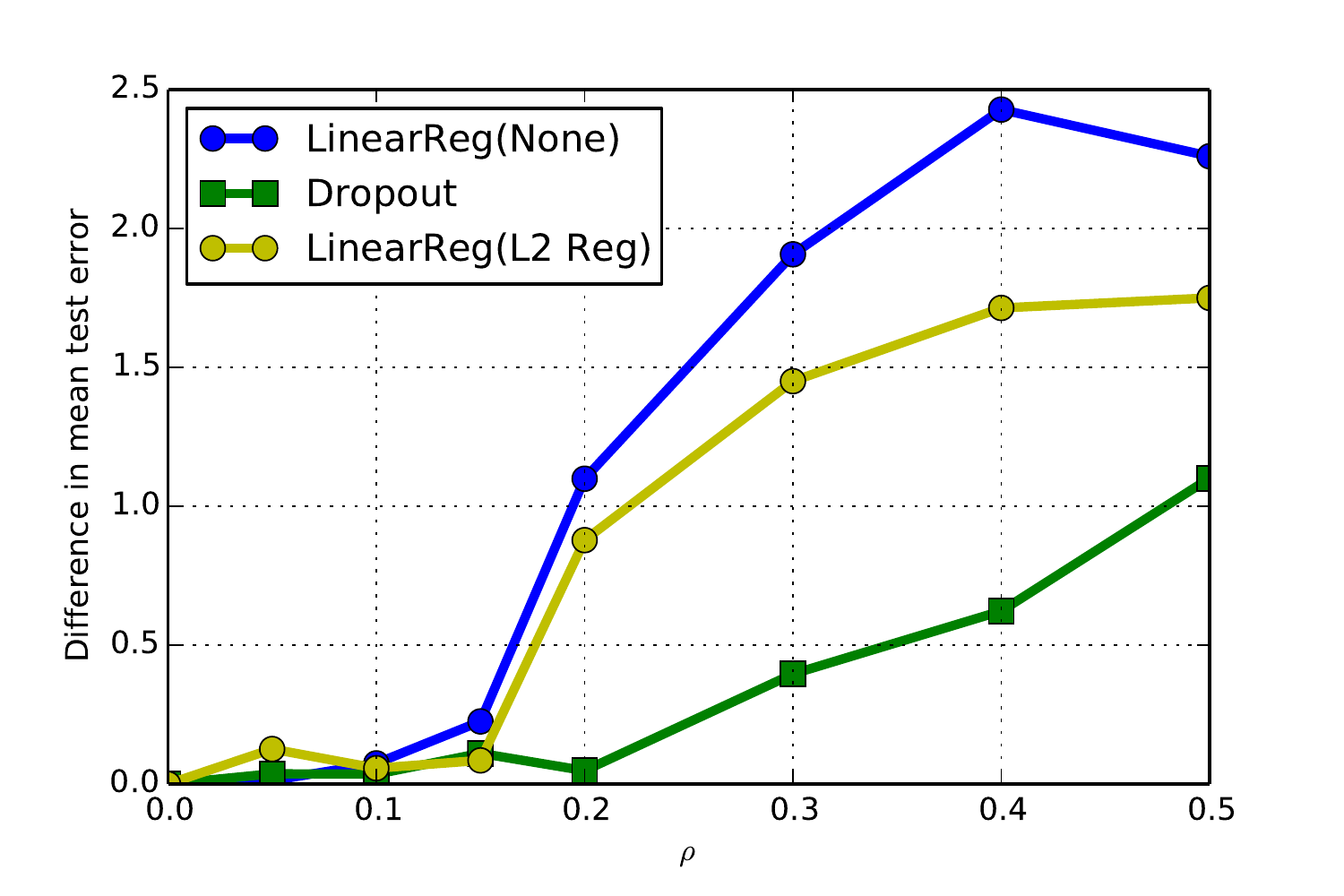}&
    \hspace*{-25pt}\includegraphics[width=0.28\textwidth]{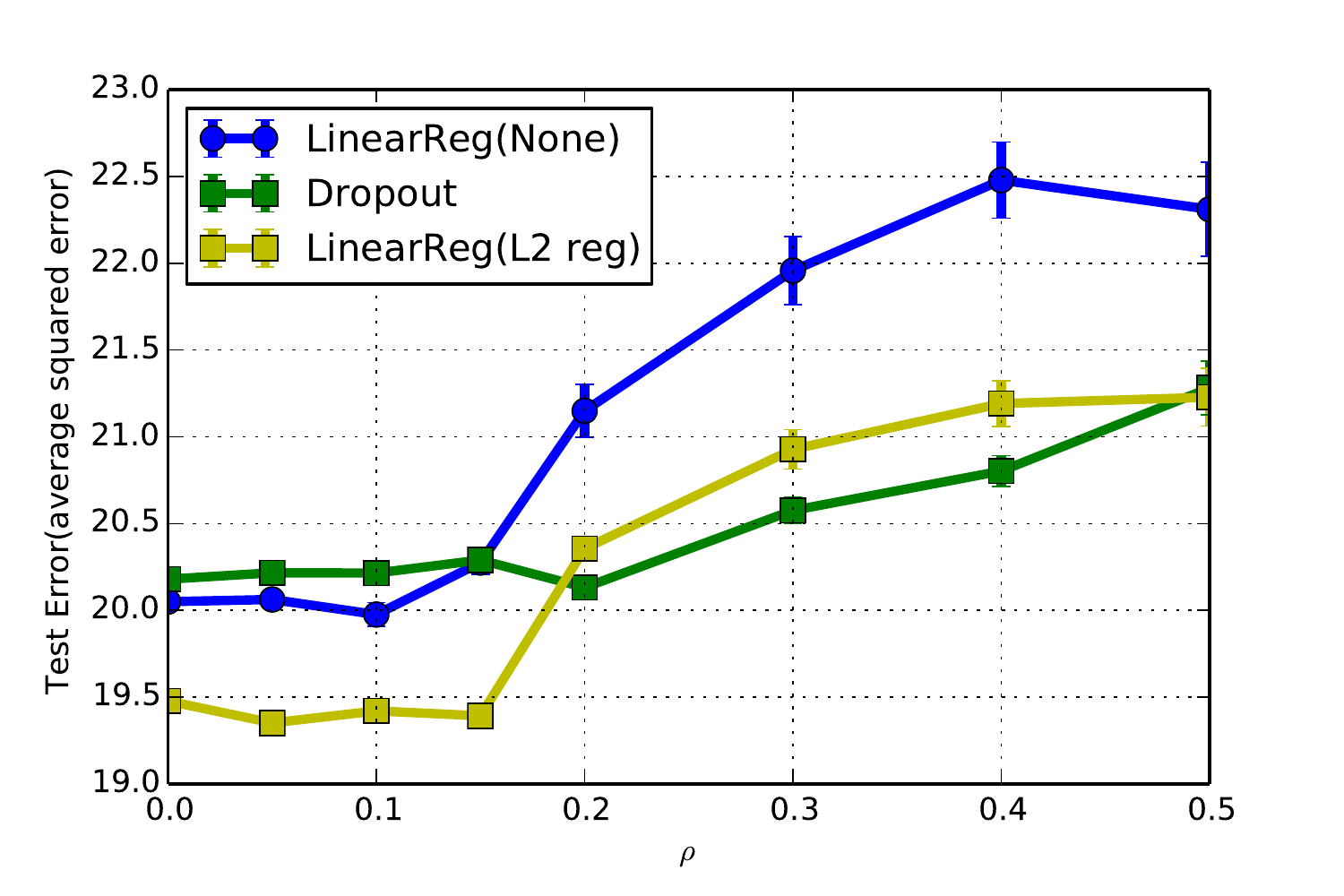}&
    \hspace*{-25pt}\includegraphics[width=0.28\textwidth]{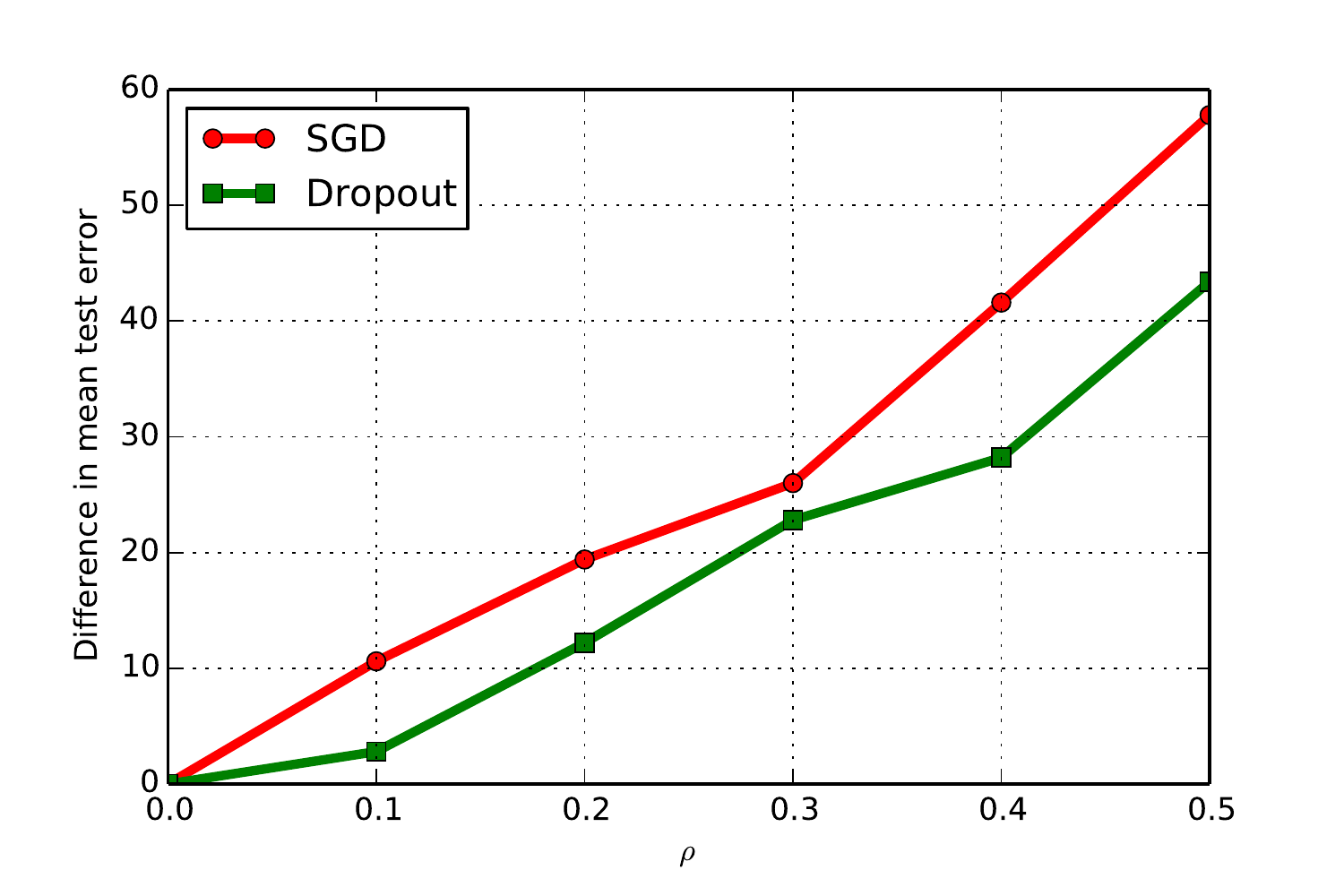}&
    \hspace*{-25pt}\includegraphics[width=0.28\textwidth]{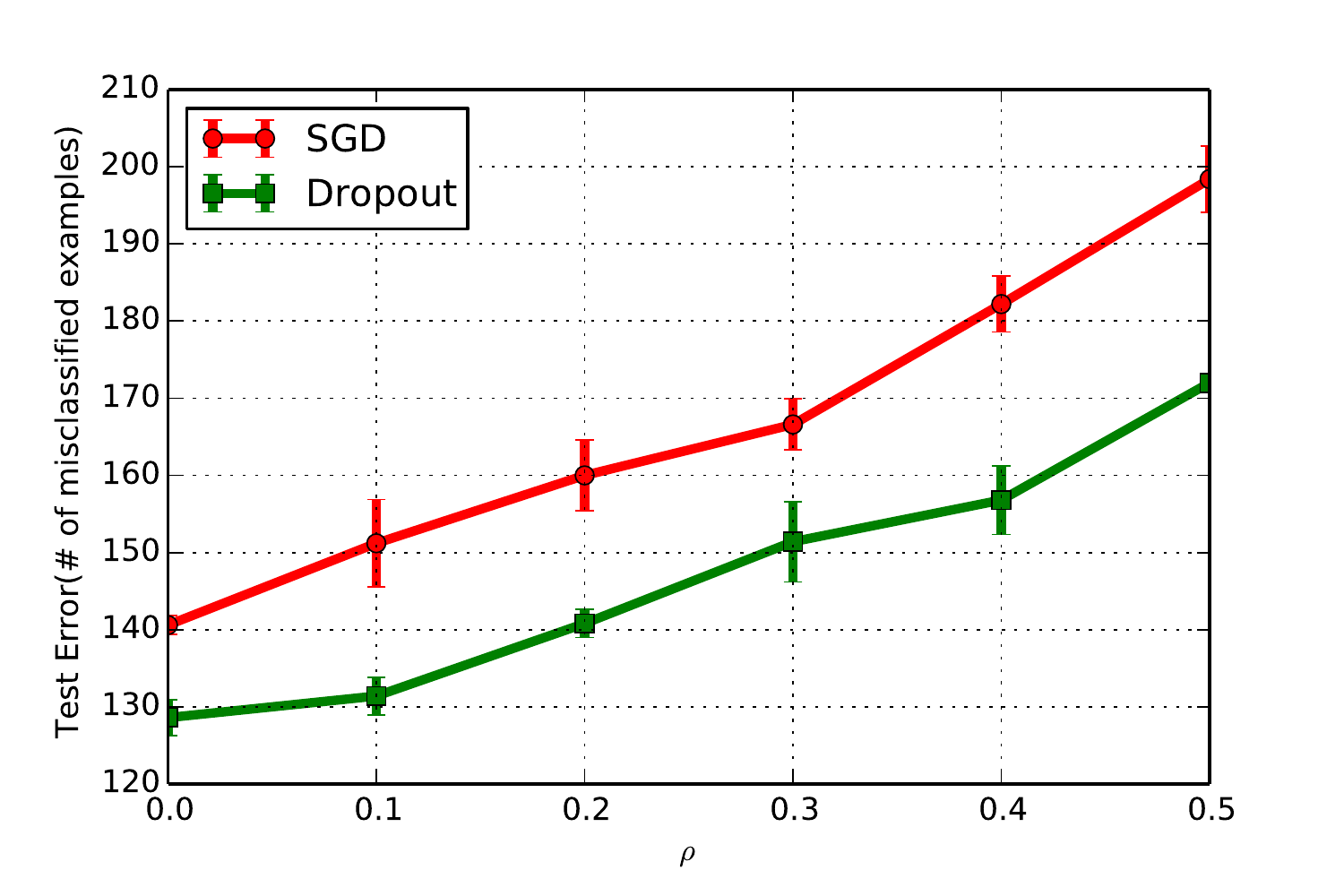}\\
(a)&(b)&(c)&(d)
\end{tabular}
\caption{(a), (b): Stability of linear regression with Boston housing data set (random sub-sampling). (c), (d): Stability analysis for DBN with MNIST data set (under random sub-sampling).}
\label{fig:dbn_fig}
\end{figure*}
In this section, we provide experimental evidence to support the stability guarantees we provided for dropout in Section \ref{sec:stabDropout} (for more extensive results, refer Appendix \ref{app:experiment}).
We empirically measure stability by observing the effect on the performance of the learning algorithm, as a function of the fraction of training examples removed.
This measure captures how dependent an algorithm is on a particular subset of the training data.
We show results for GLMs as well as for deep belief networks (DBN's). We compare against the following two baseline methods (wherever applicable): a) unregularized models and b) $L_2$-regularized GLM's. We describe our experimental setup and results for each of these model classes below.

{\bf Stability of dropout for logistic regression}\label{sec:exptGLM}
We introduce perturbations of two forms: a) \emph{random} removal of training examples and b) \emph{adversarially} remove training examples.

{\em Random removal of training examples}: For a given $\rho\in[0,1]$, we train a model on a randomly selected $(1-\rho)$-fraction of the training data. We report the test error and the {\em difference in mean test error} which is the absolute difference between the test error and the baseline error (the test error obtained by using the complete training dataset). We refer to this difference as the \emph{marginal error}.

We present results on the benchmark Atheist dataset from the 20 newsgroup corpus; total number of examples = 1427, dimensionality: 22178.
We use $50\%$ of the data for training and the remaining for testing and use a dropout rate of $0.5$. We measure the error in terms of fraction of misclassified examples.
Figure \ref{fig:stability}(a,b) shows the results for different values of $\rho$ when training a logistic regression model with no regularization, $L_2$ regularization, and two variants of dropout: ``standard'' dropout \cite{hinton2012improving}, and \emph{deterministic dropout} outlined in \cite{fastdropout}.

We observe that the dropout variants exhibit more stability than the unregularized or the $L_2$ regularized versions. Moreover, deterministic dropout is more stable than standard dropout. Notice that, even though dropout consistently has a lower test error than other methods, its effectiveness diminishes with increasing $\rho$. We hypothesize that with decreasing amount of training data, the regularization provided by dropout also decreases (see Section \ref{sec:dropoutGen} and Appendix \ref{sec:ERM}).

{\em Adversarial removal of training examples}: Let $D={(x_1,y_1),\cdots,(x_n,y_n)}$ be a given training set. Let $\theta_{\sf full}$ be a model learned on the complete set $D$. For a given value of $\rho$, we remove the $n\rho$ samples $x\in D$ which have minimum $|\ip{x}{\theta_{\sf full}}|$. The rest of the experiment remains the same as in the random removal setting. Figure \ref{fig:stability}(c,d) shows the test error and the marginal error for different regularization methods w.r.t.~$\rho$ in this adversarial setting.

As with random removal, dropout continues to be at least as good as the other regularization methods studied. However, when $\rho>0.5$ observe that dropout's advantage decreases very rapidly, and all the methods tend to perform similarly. 

{\bf Stability of linear regression}\label{sec:stabLinExpt}
Next, we apply our methods to linear regression using the Boston housing dataset \cite{Bache+Lichman:2013} (with 506 samples and 14 features) for our experiments. 
We use 300 examples for training and the rest for testing. Figure \ref{fig:dbn_fig} (a), (b) shows that the marginal error of dropout is less than that of the other methods for all values of $\rho$. Interestingly, for small values of $\rho$, dropout performs worse than $L_2$ regularization, although it performs better at higher values. Here we use a dropout rate of $0.05$, and we measure the mean squared error.


{\bf Stability of deep belief networks}:
\label{sec:exptDBN}
While our theoretical stability guarantees hold only for generalized linear models, our experiments indicate that they extend to deep belief networks (DBN) too. We posit that the dropout algorithm on DBN's (after pre-training) operates in a locally convex region, where the stability properties should hold. 

We use the MNIST data set for our DBN experiments. Experiments with other data sets are in Appendix \ref{app:experimentDBN}. MNIST dataset contains 60000 examples for training and 10000 for testing. For training a DBN on this data set, we use a network with four layers\footnote{We use the \emph{gdbn} and \emph{nolearn} python toolkits for training a DBN.}. We use 784, 800, 800, and 10 units in each layer respectively. Our error measure is the \# of misclassifications.

As in the previous experiments, we measure stability by randomly removing training examples. See Figure \ref{fig:dbn_fig} (c), (d) for test error and marginal error of dropout as well as the standard SGD algorithm applied to DBNs. Similar to the GLM setting, we observe that dropout exhibits more stability and accuracy than the unregularized SGD procedure. In fact for the case of 50\% training data, dropout is $16\%$ more accurate than SGD. 


\newpage
\small
\bibliographystyle{plainnat}
\bibliography{reference}

\begin{thebibliography}{38}
\providecommand{\natexlab}[1]{#1}
\providecommand{\url}[1]{\texttt{#1}}
\expandafter\ifx\csname urlstyle\endcsname\relax
  \providecommand{\doi}[1]{doi: #1}\else
  \providecommand{\doi}{doi: \begingroup \urlstyle{rm}\Url}\fi

\bibitem[Andoni et~al.(2014)Andoni, Panigrahy, Valiant, and Zhang]{AndoniPV014}
Alexandr Andoni, Rina Panigrahy, Gregory Valiant, and Li~Zhang.
\newblock Learning polynomials with neural networks.
\newblock In \emph{ICML}, 2014.

\bibitem[Bache and Lichman(2013)]{Bache+Lichman:2013}
K.~Bache and M.~Lichman.
\newblock {UCI} machine learning repository, 2013.
\newblock URL \url{http://archive.ics.uci.edu/ml}.

\bibitem[Baldi and Sadowski(2014)]{baldi2014dropout}
Pierre Baldi and Peter Sadowski.
\newblock The dropout learning algorithm.
\newblock \emph{Artificial intelligence}, 2014.

\bibitem[Bassily et~al.(2014)Bassily, Smith, and Thakurta]{BST}
Raef~E Bassily, Adam Smith, and Abhradeep Thakurta.
\newblock Private empirical risk minimization, revisited.
\newblock \emph{Personal communication}, 2014.

\bibitem[Bhaskar et~al.(2010)Bhaskar, Laxman, Smith, and Thakurta]{BLST}
Raghav Bhaskar, Srivatsan Laxman, Adam Smith, and Abhradeep Thakurta.
\newblock Discovering frequent patterns in sensitive data.
\newblock In \emph{KDD}, 2010.

\bibitem[Bousquet and Elisseeff(2002)]{BousquetElisseeff-2002}
Olivier Bousquet and Andr\'e Elisseeff.
\newblock Stability and generalization.
\newblock \emph{Journal of Machine Learning Research}, 2:\penalty0 499 -- 526,
  2002.

\bibitem[Bousquet et~al.(2004)Bousquet, Boucheron, and
  Lugosi]{bousquet2004introductionA}
Olivier Bousquet, St{\'e}phane Boucheron, and G{\'a}bor Lugosi.
\newblock Introduction to statistical learning theory.
\newblock In \emph{Advanced Lectures on Machine Learning}. 2004.

\bibitem[Chaudhuri and Monteleoni(2008)]{CM08}
Kamalika Chaudhuri and Claire Monteleoni.
\newblock Privacy-preserving logistic regression.
\newblock In Daphne Koller, Dale Schuurmans, Yoshua Bengio, and L{\'e}on
  Bottou, editors, \emph{NIPS}. MIT Press, 2008.

\bibitem[Chaudhuri et~al.(2011)Chaudhuri, Monteleoni, and Sarwate]{CMS11}
Kamalika Chaudhuri, Claire Monteleoni, and Anand~D. Sarwate.
\newblock Differentially private empirical risk minimization.
\newblock \emph{JMLR}, 12:\penalty0 1069--1109, 2011.

\bibitem[Duchi et~al.(2013)Duchi, Jordan, and Wainwright]{DuchiJW13}
John~C. Duchi, Michael~I. Jordan, and Martin~J. Wainwright.
\newblock Local privacy and statistical minimax rates.
\newblock In \emph{FOCS}, 2013.

\bibitem[Dwork(2006)]{Dwork06}
Cynthia Dwork.
\newblock Differential privacy.
\newblock In \emph{ICALP}, LNCS, pages 1--12, 2006.

\bibitem[Dwork(2008)]{Dwork08}
Cynthia Dwork.
\newblock Differential privacy: A survey of results.
\newblock In \emph{TAMC}, pages 1--19. Springer, 2008.

\bibitem[Dwork and Lei(2009)]{DL09}
Cynthia Dwork and Jing Lei.
\newblock Differential privacy and robust statistics.
\newblock In \emph{STOC}, 2009.

\bibitem[Dwork et~al.(2006{\natexlab{a}})Dwork, Kenthapadi, Mcsherry, Mironov,
  and Naor]{ODO}
Cynthia Dwork, Krishnaram Kenthapadi, Frank Mcsherry, Ilya Mironov, and Moni
  Naor.
\newblock Our data, ourselves: Privacy via distributed noise generation.
\newblock In \emph{In EUROCRYPT}, pages 486--503. Springer, 2006{\natexlab{a}}.

\bibitem[Dwork et~al.(2006{\natexlab{b}})Dwork, McSherry, Nissim, and
  Smith]{DMNS06}
Cynthia Dwork, Frank McSherry, Kobbi Nissim, and Adam Smith.
\newblock Calibrating noise to sensitivity in private data analysis.
\newblock In \emph{TCC}, 2006{\natexlab{b}}.

\bibitem[Helmbold and Long(2014)]{HelmboldL14}
David~P. Helmbold and Philip~M. Long.
\newblock On the inductive bias of dropout.
\newblock \emph{CoRR}, abs/1412.4736, 2014.
\newblock URL \url{http://arxiv.org/abs/1412.4736}.

\bibitem[Hinton et~al.(2012)Hinton, Srivastava, Krizhevsky, Sutskever, and
  Salakhutdinov]{hinton2012improving}
Geoffrey~E Hinton, Nitish Srivastava, Alex Krizhevsky, Ilya Sutskever, and
  Ruslan~R Salakhutdinov.
\newblock Improving neural networks by preventing co-adaptation of feature
  detectors.
\newblock \emph{arXiv preprint arXiv:1207.0580}, 2012.

\bibitem[Jain and Thakurta(2014)]{jain2014near}
Prateek Jain and Abhradeep~Guha Thakurta.
\newblock (near) dimension independent risk bounds for differentially private
  learning.
\newblock In \emph{ICML}, 2014.

\bibitem[Jain et~al.(2012)Jain, Kothari, and Thakurta]{JKT12}
Prateek Jain, Pravesh Kothari, and Abhradeep Thakurta.
\newblock Differentially private online learning.
\newblock In \emph{COLT}, 2012.

\bibitem[Kalai and Vempala(2005)]{KV05}
Adam Kalai and Santosh Vempala.
\newblock Efficient algorithms for online decision problems.
\newblock \emph{Journal of Computer and System Sciences}, 2005.

\bibitem[Kifer et~al.(2012)Kifer, Smith, and Thakurta]{KST12}
Daniel Kifer, Adam Smith, and Abhradeep Thakurta.
\newblock Private convex empirical risk minimization and high-dimensional
  regression.
\newblock In \emph{COLT}, 2012.

\bibitem[Maaten et~al.(2013)Maaten, Chen, Tyree, and
  Weinberger]{maaten2013learning}
Laurens Maaten, Minmin Chen, Stephen Tyree, and Kilian~Q Weinberger.
\newblock Learning with marginalized corrupted features.
\newblock In \emph{Proceedings of the 30th International Conference on Machine
  Learning (ICML-13)}, pages 410--418, 2013.

\bibitem[McAllester(2013)]{mcallester2013pac}
David McAllester.
\newblock A pac-bayesian tutorial with a dropout bound.
\newblock \emph{arXiv preprint arXiv:1307.2118}, 2013.

\bibitem[Nikolov et~al.(2013)Nikolov, Talwar, and Zhang]{nikolov2013geometry}
Aleksandar Nikolov, Kunal Talwar, and Li~Zhang.
\newblock The geometry of differential privacy: the sparse and approximate
  cases.
\newblock In \emph{STOC}, 2013.

\bibitem[Poggio et~al.(2011)Poggio, Voinea, and Rosasco]{PogVoiRos11}
Tomaso Poggio, Stephen Voinea, and Lorenzo Rosasco.
\newblock Online learning, stability, and stochastic gradient descent.
\newblock \emph{CoRR}, abs/1105.4701, 2011.

\bibitem[Shalev-Shwartz et~al.(2009)Shalev-Shwartz, Shamir, Srebro, and
  Sridharan]{SSSS}
Shai Shalev-Shwartz, Ohad Shamir, Nathan Srebro, and Karthik Sridharan.
\newblock {Stochastic Convex Optimization}.
\newblock In \emph{Proceedings of the Conference on Learning Theory (COLT)},
  2009.

\bibitem[Shalev-Shwartz et~al.(2010)Shalev-Shwartz, Shamir, Srebro, and
  Sridharan]{shalev2010learnability}
Shai Shalev-Shwartz, Ohad Shamir, Nathan Srebro, and Karthik Sridharan.
\newblock Learnability, stability and uniform convergence.
\newblock \emph{JMLR}, 2010.

\bibitem[Shamir and Zhang(2013)]{Ishamir13}
Ohad Shamir and Tong Zhang.
\newblock Stochastic gradient descent for non-smooth optimization: Convergence
  results and optimal averaging schemes.
\newblock In \emph{ICML}, 2013.

\bibitem[Smith and Thakurta(2013)]{STLasso13}
Adam~D. Smith and Abhradeep Thakurta.
\newblock Differentially private model selection via stability arguments and
  the robustness of the lasso.
\newblock In \emph{COLT}, 2013.

\bibitem[Song et~al.(2013)Song, Chaudhuri, and Sarwate]{song2013stochastic}
Shuang Song, Kamalika Chaudhuri, and Anand~D Sarwate.
\newblock Stochastic gradient descent with differentially private updates.
\newblock In \emph{IEEE Global Conference on Signal and Information
  Processing}, 2013.

\bibitem[Sridharan et~al.(2008)Sridharan, Shalev-shwartz, and
  Srebro]{sridharan2008fast}
Karthik Sridharan, Shai Shalev-shwartz, and Nathan Srebro.
\newblock Fast rates for regularized objectives.
\newblock In \emph{NIPS}, 2008.

\bibitem[Szegedy et~al.(2013)Szegedy, Zaremba, Sutskever, Bruna, Erhan,
  Goodfellow, and Fergus]{SzegedyZSBEGF13}
Christian Szegedy, Wojciech Zaremba, Ilya Sutskever, Joan Bruna, Dumitru Erhan,
  Ian~J. Goodfellow, and Rob Fergus.
\newblock Intriguing properties of neural networks.
\newblock \emph{CoRR}, abs/1312.6199, 2013.

\bibitem[Thakurta(2015)]{Thakurta15}
Abhradeep Thakurta.
\newblock Beyond worst case sensitivity in private data analysis.
\newblock In \emph{Encyclopedia of Algorithms}. 2015.

\bibitem[van Erven et~al.(2014)van Erven, Kot{\l}owski, and
  Warmuth]{van2014follow}
Tim van Erven, Wojciech Kot{\l}owski, and Manfred~K Warmuth.
\newblock Follow the leader with dropout perturbations.
\newblock In \emph{COLT}, 2014.

\bibitem[Wager et~al.(2013)Wager, Wang, and Liang]{wager2013dropout}
Stefan Wager, Sida Wang, and Percy Liang.
\newblock Dropout training as adaptive regularization.
\newblock In \emph{NIPS}, 2013.

\bibitem[Wager et~al.(2014)Wager, Fithian, Wang, and Liang]{wager2014altitude}
Stefan Wager, William Fithian, Sida Wang, and Percy~S Liang.
\newblock Altitude training: Strong bounds for single-layer dropout.
\newblock In \emph{NIPS}, 2014.

\bibitem[Wang and Manning(2013)]{fastdropout}
Sida Wang and Christopher~D. Manning.
\newblock Fast dropout training.
\newblock In \emph{ICML (2)}, 2013.

\bibitem[Wang et~al.(2013)Wang, Wang, Wager, Liang, and
  Manning]{wang2013feature}
Sida Wang, Mengqiu Wang, Stefan Wager, Percy Liang, and Christopher~D Manning.
\newblock Feature noising for log-linear structured prediction.
\newblock In \emph{EMNLP}, pages 1170--1179, 2013.

\end{thebibliography}
\normalsize
\newpage
\onecolumn
\appendix
\section{Fast rates of convergence for dropout optimization}
\label{app:fastDrop}

\subsection{Empirical risk minimization (ERM) formulation of dropout} \label{sec:ERM}
For simplicity of exposition, we modify the ERM formulation to incorporate dropout perturbation as a part of the optimization problem itself. We stress that the ERM formulation is for intuition only. In Section \ref{sec:dropoutGen}, we analyze the stochastic gradient descent (SGD) variant  of the dropout heuristic and show that the excess risk bound for the SGD variant is similar to that of the ERM variant.

Given a loss function $\ell$, convex set $\C$, and data set $D=\{(x_1,y_1),\cdots,(x_n,y_n)\}$ which consists of $n$ i.i.d. samples drawn from $\tau(\D)$, fitting a model with dropout corresponds to the following optimization:
\begin{equation}\label{eq:dropout}
	\nptheta=\arg\min\limits_{\theta\in\C}\frac{1}{n}\sum\limits_{i=1}^n \ell(2\ip{x_i * b_i}{\theta};y_i),
\end{equation}
where each $b_i$ is an i.i.d.~sample drawn uniformly from $\{0,1\}^p$, and the operator $*$ refers to the Hadamard product. We assume that the loss function $\ell(u, y):\R^2\rightarrow \R$ is strongly convex in $u$. For example, in the case of least-squares linear regression the loss function $\ell(2\ip{x}{\theta};y)$ is $(y-2\ip{x}{\theta})^2$.

\begin{lem}
	Let $b$ be drawn uniformly from $\{0,1\}^p$ and let the expected population risk be given by
	\begin{equation}\label{eq:risk12}
		{\sf Risk}(\theta)=\E_{(x,y)\sim\tau(\D), b}[\ell(2\ip{x*b}{\theta};y)].	
	\end{equation}
	Let $\ell(u, y)$ be a $\alpha$-strongly convex function w.r.t.~$u$. Then, the expected population risk \eqref{eq:risk12} is $\alpha \cdot \Delta$ strongly convex w.r.t.~$\theta$, where $\Delta=\min\limits_{j\in[p]}\E\limits_{x\sim\tau(\D)}[x(j)^2]$ and $x(j)$ is the $j$-th coordinate of $x$. \label{lem:risk_sc}
\end{lem}

\begin{proof}
	Now,
	\begin{align*}\nabla^2_\theta {\sf Risk}(\theta)&=\E_{(x,y)\sim\tau(\D)}\left[\E_{b}\left[4\frac{\partial^2 \ell(2\ip{x*b}{\theta};y)}{\partial \ip{2x*b}{\theta}^2}(x*b)(x*b)^T\right]\right] \nonumber\\
	&\succeq 4\alpha\E_{(x,y)\sim\tau(\D)}\left[\frac{1}{4}\diag(xx^T)+\frac{1}{4}xx^T\right]\succeq \alpha \Delta,\end{align*}
	where second to last inequality follows by strong convexity of $\ell$ and from the fact that $b$ is sampled uniformly from $\{0, 1\}^{p}$.
\end{proof}
An immediate corollary to the above lemma is that for normalized features, i.e., $\E[\sum_j x(j)^2]=1$, the dropout risk function \eqref{eq:risk12} is the same as that for $L_2$ regularized least squares (in expectation).
\begin{corollary}
  Let $b$ be drawn uniformly from $\{0,1\}^p$ and let $\ell(2\ip{x*b}{\theta};y)$ be the least squares loss function, i.e., $\ell(2\ip{x}{\theta};y)=(2\ip{x*b}{\theta}-y)^2$. Then,
  \begin{equation}\label{eq:linearRisk}
  	{\sf Risk}(\theta)=\E\limits_{(x,y)\sim\tau(\D)}\left[(y-\ip{x}{\theta})^2\right]+\theta^T \E\limits_{(x,y)\sim\tau(\D)}\left[\diag(xx^T)\right]\theta.
  \end{equation}
\end{corollary}

Next, we provide an excess risk bound for $\nptheta$, the optimal solution to the dropout-based ERM \eqref{eq:dropout}. Our proof technique closely follows that of \cite{sridharan2008fast} and crucially uses the fact that \cite{sridharan2008fast} only requires strong convexity of the expected  loss function. Below we provide the risk bound.

\begin{thm}[Dropout generalization bound]
Let $\C\subseteq\re^p$ be a fixed convex set and let Assumption \ref{assump:normDat} be true for the data domain $\D$ and the loss $\ell$. Let $D=\{(x_1,y_1),\cdots,(x_n,y_n)\}$ be $n$ i.i.d. samples drawn from $\tau(\D)$. Let $V=\{b_1,\cdots,b_n\}$ be $n$ i.i.d. vectors drawn uniformly from $\{0,1\}^p$. Let $\nptheta=\arg\min\limits_{\theta\in\C}\frac{1}{n}\sum\limits_{i=1}^n \ell(2\ip{x_i * b_i}{\theta};y_i)$ and let $\rL(\theta)$ be defined as in \eqref{eq:risk12}.
Then, w.p. $\geq 1-\gamma$ (over the randomness of both $D$ and $V$), we have the following:
$$\E\left[\rL(\nptheta)-\min\limits_{\theta\in\C}\rL(\theta)\right]=O\left(\frac{(GB)^2\log(1/\gamma)}{\Delta n}\right).$$
Here $\Delta=\min\limits_{j\in[p]}{\E\limits_x\left[x(j)^2\right]}$, and the parameters $G$ and $B$ are defined in Assumption \ref{assump:normDat}.
\label{thm:dropoutGen}
\end{thm}

\begin{proof}
Define $g_\theta$ as: $g_\theta(x,y,b)=\ell(2\ip{x*b}{\theta};y)-\ell(2\ip{x*b}{\theta^*};y)$, where $\theta^*=\arg\min\limits_{\theta\in\C}\E\limits_{(x,b,y)}\left[\ell(2\ip{x*b}{\theta};y)\right]$. Also, let $\mathcal{G}_\theta=\{g_\theta:\theta\in\C\}$.
Following the technique of \cite{sridharan2008fast}, we will now scale each of the $g_\theta$'s such that the ones which have higher expected value over $(x,y,b)$ have exponentially smaller weight. This helps us obtain a more fine-grained bound on the Rademacher complexity, which will be apparent below.

Let $\mathcal{G}_a=\left\{g_\theta^a=\frac{g_\theta}{4^{k_a(\theta)}}:\theta\in\C,k_a(\theta)=\min\left\{k'\in\mathbb{Z}^+:\E\limits_{(x,y,b)}[g_\theta]\leq a4^{k'}\right\}\right\}$. Using standard Rademacher complexity bounds \citet[Theorem 5]{bousquet2004introductionA}, for any $\theta\in\C$, the following holds (w.p. $\geq 1-\gamma$ over the randomness in selection of dataset $H$): 
\begin{equation}
\sup\limits_{q\in\mathcal{G}_a}\left[\E\limits_{(x,y,b)}\left[q(x,y,b)\right]-\frac{1}{n}\sum\limits_{i=1}^n q(x_i,y_i,b_i)\right]\leq 2\mathcal{R}(\mathcal{G}_a)+\left(\sup\limits_{q\in\mathcal{G}_a,(x,y,b)}|q(x,y,b)|\right)\sqrt\frac{\log(1/\gamma)}{2n}
\label{eq:l123}
\end{equation}
Here $\mathcal{R}$ refers to the Rademacher complexity of the hypothesis class. In the following we will bound each of the term in the right hand side of \eqref{eq:l123}.

\begin{lem}
Let $\Delta=\min\limits_{j\in[p]}{\E\limits_x\left[x(j)^2\right]}$, where $x(j)$ refers to the $j$-th coordinate of $x$. We claim that
$$\sup\limits_{q\in\mathcal{G}_a,(x,y,b)}|q(x,y,b)|\leq 2 (GB)\sqrt\frac{2a}{\Delta}.$$
\label{lem:121}
\end{lem}
\begin{proof}
By the definition of the bound on the domain of $x$ and assumption on $\ell$, we have $\ltwo{\grad_\theta \ell(2\ip{x*b}{\theta};y)}\leq 2GB$. Therefore $\forall\theta\in\C, q\in\mathcal{G}_a, (x,y)\in\D, b\in\{0,1\}^p$, we have the following.
\begin{equation}
|q(x,y,b)|\leq\frac{|g_\theta(x,y,b)|}{4^{k_a(\theta)}}\leq \frac{2GB\ltwo{\theta-\theta^*}}{4^{k_a(\theta)}}.
\label{eq:1223a}
\end{equation}
In the following we now bound $\ltwo{\theta-\theta^*}$. Using Lemma~\ref{lem:risk_sc}, $\E\limits_{x, y, b}[ \ell(2\ip{x*b}{\theta};y)]$ is strongly convex and hence using optimality of $\theta^*$, we have: 
\begin{equation}
\ltwo{\theta-\theta^*}\leq\sqrt{\frac{2}{\Delta}\E\limits_{x,y,b}[g_\theta(x,y,b)]}\leq \sqrt{\frac{2a\cdot 4^{k_a(\theta)}}{\Delta}},
\label{eq:12a}
\end{equation}
where the last equations follows using definition of $k_a(\theta)$. 
\end{proof}
Now, directly using \citet[Lemma 7]{sridharan2008fast}, we can bound the Rademacher complexity in \eqref{eq:l123} by $\mathcal{R}(\mathcal{G}_a)\leq 4\sqrt\frac{128a}{\Delta n}$. Therefore we can bound \ref{eq:l123} as follows.
\begin{equation}
\sup\limits_{q\in\mathcal{G}_a}\left[\E\limits_{(x,y,b)}\left[q(x,y,b)\right]-\frac{1}{n}\sum\limits_{i=1}^n q(x_i,y_i,b_i)\right]=O\left(GB\sqrt\frac{a\log(1/\gamma)}{\Delta n}\right).
\label{eq:l123a}
\end{equation}
Now notice that for any $a>0$, w.p. at least $1-\gamma$ we have the following from \eqref{eq:l123a}.
\begin{align}
\E\limits_{x,y,b}\left[g_\theta(x,y,b)\right]-\frac{1}{n}\sum\limits_{i=1}^n g_\theta (x_i,y_i,b_i)&=4^{k_a(\theta)}\left[\E\limits_{x,y,b}\left[g^a_\theta(x,y,b)\right]-\frac{1}{n}\sum\limits_{i=1}^n g^a_\theta (x_i,y_i,b_i)\right]\nonumber\\
&=4^{k_a(\theta)}O\left(GB\sqrt\frac{a\log(1/\gamma)}{\Delta n}\right).
\label{eq:al123}
\end{align}
When ${k_a(\theta)}=0$, \eqref{eq:al123} implies the following.
\begin{equation}
\E\limits_{x,y,b}\left[g_\theta(x,y,b)\right]-\frac{1}{n}\sum\limits_{i=1}^n g_\theta (x_i,y_i,b_i)=O\left(GB\sqrt\frac{a\log(1/\gamma)}{\Delta n}\right).
\label{eq:ah12}
\end{equation}
When $k_a(\theta)>0$, one has $4^{k_a(\theta)-1}a<\E\limits_{x,y,b}\left[g_\theta(x,y,b)\right]$. Substituting this in \eqref{eq:al123} and by rearranging the terms, we have the following. Here $M=\xi GB\sqrt\frac{a\log(1/\gamma)}{\Delta n}$ for some constant $\xi>0$.
\begin{equation}
\E\limits_{x,y,b}\left[g_\theta(x,y,b)\right]\leq\frac{1}{1-4M/a}\left[\frac{1}{n}\sum\limits_{i=1}^n g_\theta (x_i,y_i,b_i)\right].
\label{eq:op123}
\end{equation}
Setting $a=8M$, and combining \eqref{eq:op123} and \eqref{eq:ah12} for the cases $k_a(\theta)=0$ and $k_a(\theta)>0$ completes the proof.
\end{proof}

\subsection{Proof of Theorem \ref{thm:dropoutGenSGD} (Generalization bound for dropout gradient descent)}
\label{app:dropoutGenSGD}

\begin{proof}
Let
\begin{equation}
J(\theta;D)=\frac{1}{n}\sum\limits_{i=1}^n\E\limits_{b\sim\{0,1\}^p}\left[\ell(2\ip{x_i*b}{\theta};y_i)\right].
\label{eq:J123}
\end{equation}
Let $\ell_t(\theta)=\ell(2\ip{\hat x_t*b_t}{\theta};\hat y_t)$ for the ease of notation, where $x_t,y_t$ and $b_t$ are the parameters used in the $t$-th iterate. Over the randomness of the SGD algorithm, we have the following:
\begin{equation}
\E[\grad \ell_t(\theta)]=\grad J(\theta;D).
\label{eq:gradF}
\end{equation}
Additionally, we have the following.
\begin{align}
\nabla^2_\theta J(\theta;D)&=\frac{1}{n}\sum\limits_{i=1}^n\E_{b}\left[4\frac{\partial^2 \ell(2\ip{x_i*b}{\theta};y_i)}{\partial \ip{2x_i*b}{\theta}^2}(x_i*b)(x_i*b)^T\right] \nonumber\\
&\succeq 4\frac{1}{n}\sum\limits_{i=1}^n\left[\frac{1}{4}\diag(x_ix_i^T)+\frac{1}{4}x_ix_i^T\right]\succeq \Delta_1.
\label{eq:akjhd123}
\end{align}
\eqref{eq:akjhd123} implies that $J(\theta;D)$ is $\Delta_1$-strongly convex. Using Theorem \ref{thm:stocGradConv} we obtain the following:

\begin{equation}
\E\left[J(\theta_T;D)-\min\limits_{\theta\in\C} J(\theta;D)\right]=O\left(\frac{(GB)^2\log T}{\Delta_1 T}\right).
\label{eq:ambc1321}
\end{equation}
where the expectation is over the randomness of the SGD algorithm.

Notice that by the definition of $J(\theta;D)$ in \eqref{eq:J123}, we have
$$\rL(\theta)=\E\limits_{D\sim\tau(\D)^n}\left[J(\theta;D)\right].$$
Theorem now follows by using \eqref{eq:ambc1321} and Theorem \ref{thm:fastStrongObj}.
\end{proof}

\begin{thm}[Convergence of expected stochastic gradient descent \cite{Ishamir13}]
Suppose $J(\theta)$ is a is $\Delta_1$-strongly convex function, and let the stochastic gradient descent algorithm be $\theta_{t+1}=\proj_\C(\theta_t-\eta_t\grad \ell_t(\theta_t))$ such that $\E[\grad\ell_t(\theta)]=\grad J(\theta)$ for all $\theta\in\C$. Additionally, assume that $\E[\ltwo{\grad{\ell_t(\theta_t)}}^2]=O(L^2)$ for all $t$. If the learning rate $\eta_t=1/(\Delta_1 t)$, then for any $T>1$ the following is true.
$$\E\left[J(\theta)-\min\limits_{\theta\in\C}J(\theta)\right]=O\left(\frac{L^2\log T}{\Delta_1 T}\right).$$
All the expectations are over the randomness of the stochastic gradient descent algorithm.
\label{thm:stocGradConv}
\end{thm}

Now in the following we state a variant of \citet[Theorem 1]{sridharan2008fast}. The only difference is that we use the strong convexity of $\rL(\theta)$ in the guarantee instead of the strong convexity of $J(\theta;D)$. The proof of this variant is exactly the same as \citet[Theorem 1]{sridharan2008fast}.

\begin{thm}[Fast convergence for strongly convex objective \cite{sridharan2008fast}]
Let $D\sim\tau(\D)^n$. Over the randomness of the data distribution $\tau(\D)$, for all $\theta\in\C$, the following is true w.p. $\geq 1- \gamma$.
$$\rL(\theta)-\min\limits_{\theta\in\C}\rL(\theta)=O\left(J(\theta;D)-\min\limits_{\theta\in\C}J(\theta;D)+\frac{(GB)^2\log(1/\gamma)}{\Delta n}\right).$$
Here $\Delta=\min\limits_{j\in[p]}{\E\limits_x\left[x(j)^2\right]}$.
\label{thm:fastStrongObj}
\end{thm}


\section{Differentially private learning for GLMs using dropout}
\label{app:dropGLM}

To generalize our stability result for linear losses to the generic GLM regression setting we first provide a model stability result for the dropout-based gradient descent.  At a high-level Theorem \ref{thm:modelSGD} ensures that if the dropout gradient descent (from Section \ref{sec:dropoutGen}) is executed for $T=n^2$ iterations, then by changing any one data entry in the training data set only changes the model by at most $1/n$ (in the $L_2$-norm).
\begin{thm}[Model stability of dropout gradient descent]
	Let $\C\subseteq\re^p$ be a fixed convex set and let the data domain $\D$ and loss $\ell$ satisfy Assumption \ref{assump:normDat} with parameters $G$, $B$. Let $D=\{(x_1,y_1),\cdots,(x_n,y_n)\}$ be $n$ samples from $\D$, and let $\Lambda=\frac{1}{n}\min\limits_{k\in[n]}\min\limits_{j\in[p]}\sum\limits_{i=1,i\neq k}^n x_i(j)^2$ and learning rate $\eta_t=\frac{1}{\Delta_1 t}$. Let $T$ be the number of time steps for which gradient descent is executed. Then gradient descent with dropout ensures the following property for any data set $D'$ differing in one entry from $D$.
	$$\E\limits_{\A(D)}[\ltwo{\A(D)-\A(D')}]=O\left(\frac{GB}{\Lambda}\left(\sqrt{\frac{\log T\max\{\Delta_1/\Lambda,\Lambda/\Delta_1\}}{T}}+\frac{1}{n}\right)\right)$$
	\label{thm:modelSGD}
\end{thm}

\begin{proof}
	We will follow the notation from Section \ref{sec:dropoutGen} for convenience. Recall that\\ $J(\theta;D)=\frac{1}{n}\sum\limits_{i=1}^n\E\limits_{b\sim\{0,1\}^p}\left[\ell(2\ip{x_i*b}{\theta};y_i)\right]$ and $\ell_t(\theta)=\ell(2\ip{\hat x_t*b_t}{\theta};\hat y_t)$, where $x_t,y_t$ and $b_t$ are the parameters used in the $t$-th iterate. By the same argument as in \eqref{eq:akjhd123} we have that $J(\theta;D)$ and $J(\theta;D')$ are $\Lambda$-strongly convex, where $D'$ is any neighboring data set of $D$.
	
	Let $\theta_T(D)$ and $\theta_T(D')$ the outputs of SGD on data sets $D$ and $D'$. Similarly, let $\theta^\dagger(D)=\arg\min\limits_{\theta\in\C} J(\theta;D)$ and $\theta^\dagger(D')$ be defined analogously. Using an immediate variant of  Theorem \ref{thm:stocGradConv} we conclude that for $\theta_T$ (i.e., the $\theta$ obtained after running the SGD algorithm for $T$), $\E[J(\theta_T(D);D)-J(\theta^\dagger(D);D)]=O\left(\frac{(GB)^2\log T\max\{\Delta_1/\Lambda,\Lambda/\Delta_1\}}{\Lambda T}\right)$ .
	By the $\Lambda$-strong convexity of $J(\theta;D)$ and Jennsen's inequality, we have the following:
	\begin{equation}
	\E\left[\ltwo{\theta_T(D)-\theta^\dagger(D)}\right]=O\left(\frac{GB}{\Lambda}\left(\sqrt{\frac{\log T\max\{\Delta_1/\Lambda,\Lambda/\Delta_1\}}{T}}\right)\right).
	\label{eq:09ad}
	\end{equation}
	The same bound holds for the data set $D'$ too. In order to complete the stability argument, we show that on a neighboring data set $D'$, $\theta^\dagger(D)$ does not change too much in $L_2$-norm. W.l.o.g. assume that $D$ and $D'$ differ in the $n$-th data entry. Therefore, by strong convexity and the property of the minimizers $\theta^\dagger(D)$ and $\theta^\dagger(D')$, we have the following. For brevity we represent  $f(\theta;x,y)=\frac{1}{n}\E\limits_{b\sim\{0,1\}^p}\left[\ell(2\ip{x*b}{\theta};y)\right]$.
	\begin{align}
	&J(\theta^\dagger(D');D)\geq J(\theta^\dagger(D);D)+\frac{\Lambda}{2}\ltwo{\theta^\dagger(D')-\theta^\dagger(D)}^2\nonumber\\
	&\Leftrightarrow J(\theta^\dagger(D');D)+f(\theta^\dagger(D');y_n',x_n')-(J(\theta^\dagger(D);D)+f(\theta^\dagger(D);y_n',x_n'))\geq\frac{\Lambda}{2}\ltwo{\theta^\dagger(D')-\theta^\dagger(D)}^2\nonumber\\
	&+\left(f(\theta^\dagger(D');y_n',x_n')-f(\theta^\dagger(D);y_n',x_n')\right)\nonumber\\
	&\Leftrightarrow J(\theta^\dagger(D');D')-J(\theta^\dagger(D);D')\geq\frac{\Lambda}{2}\ltwo{\theta^\dagger(D')-\theta^\dagger(D)}^2+\left(f(\theta^\dagger(D');y_n',x_n')-f(\theta^\dagger(D);y_n',x_n')\right)\nonumber\\
	&\Leftrightarrow\left(f(\theta^\dagger(D);y_n',x_n')-f(\theta^\dagger(D');y_n',x_n')\right)\geq\frac{\Lambda}{2}\ltwo{\theta^\dagger(D')-\theta^\dagger(D)}^2\nonumber\\
	&\Rightarrow\ltwo{\theta^\dagger(D')-\theta^\dagger(D)}=O\left(\frac{GB}{n\Lambda}\right).
	\label{eq:loq123}
	\end{align}
	Now notice that Theorem \ref{thm:stocGradConv} is true for any data set. Therefore combining \eqref{eq:loq123} and \eqref{eq:09ad}, we have
	$\E\left[\ltwo{\theta_T(D)-\theta_T(D')}\right]=O\left(\frac{GB}{\Lambda}\left(\sqrt{\frac{\log T\max\{\Delta_1/\Lambda,\Lambda/\Delta_1\}}{T}}+\frac{1}{n}\right)\right)$.
	
	This completes the proof.
\end{proof}

Here we want to highlight two interesting properties of this theorem: i) stability guarantee holds even if all the data points ($x_i$'s) are aligned in the same direction, i.e., when the Hessian is rank deficient, and ii) the rate of stability is similar to the one that can be obtained by adding an $L_2$-regularizer \cite{CMS11,KST12}.

Unlike in Section \ref{sec:linLoss}, we are only able to show model stability for GLMs. One can see that local differential privacy cannot be achieved in general for this setting because the support of the distribution over the output model can be different for neighboring data sets and the randomness is multiplicative. However, by modifying the gradient descent algorithm in Section \ref{sec:dropoutGen} to add random perturbation one can achieve local differential privacy. Specifically we use the Gaussian mechanism from the differential privacy literature (\citet[Lemma~4]{nikolov2013geometry} and Appendix \ref{app:background}). But in order to add the Gaussian perturbation, we need Theorem \ref{thm:modelSGD} to hold w.h.p. We use boosting scheme to achieve this.

\mypar{Boosting of stability} Let $\theta_T^{(1)},\cdots,\theta_T^{(k)}$ be the models from $k$ independent runs of the dropout gradient descent algorithm (see Section~\ref{sec:dropoutGen}). Let $j^*=\arg\min\limits_{j\in[k]}J(\theta_T^{(j)};D)$, where $J(\theta;D)=\frac{1}{n}\sum\limits_{i=1}^n\E\limits_{b\sim\{0,1\}^p}\left[\ell(2\ip{x_i*b}{\theta};y_i)\right]$. If $k=\log(1/\delta)$ and $\epsilon_{\sf mod}$ equals the bound in Theorem \ref{thm:modelSGD}, then w.p. $\geq 1-\delta$, we have: $\ltwo{\theta^{(j^*)}_T(D)-\theta^{(j^*)}_T(D')}\leq \epsilon_{mod}\sqrt{\log(1/\delta)}$. Formally,

\begin{thm}[High probability model stability of dropout gradient descent]
	Let $k=\log(1/\delta)$ be the number of independent runs of stochastic gradient descent. Following the notation of Theorem \ref{thm:modelSGD}, with probability at least $1-\delta$  over the randomness of the algorithm, we have the following.
	$$\ltwo{\theta^{(j^*)}_T(D)-\theta^{(j^*)}_T(D')}=O\left(\frac{GB\sqrt{\log(1/\delta)}}{\Lambda}\left(\sqrt{\frac{\log T\max\{\Delta_1/\Lambda,\Lambda/\Delta_1\}}{T}}+\frac{1}{n}\right)\right).$$
	Here $D'$ is any neighboring data set of $D$.
	\label{thm:highProbModStab}
\end{thm}

\begin{proof}
	Let $\theta^\dagger=\arg\min\limits_{\theta\in\C}J(\theta;D)$. Notice that by Markov's inequality and the bound from Theorem \ref{thm:stocGradConv} one can conclude that w.p. at least $1/2$, for a given $j\in[k]$, $J(\theta^{(j)}_T;D)-J(\theta^\dagger;D)=O\left(\frac{(GB)^2\log T\max\{\Delta_1/\Lambda,\Lambda/\Delta_1\}}{\Lambda T}\right)$. For $k=\log(1/\delta)$ independent runs of the stochastic gradient descent, it follows that with probability at least $1-\delta$, $$J(\theta^{(j^*)}_T;D)-J(\theta^\dagger;D)=O\left(\frac{(GB)^2\log T\log(1/\delta)\max\{\Delta_1/\Lambda,\Lambda/\Delta_1\}}{\Lambda T}\right).$$ Now following the rest of the argument exactly as Theorem \ref{thm:modelSGD}, we get the required bound.
\end{proof}

Now, using standard results from differential privacy literature \cite{nikolov2013geometry}, we can show that $\hat\theta_T=\theta_T^{(j^*)}+\mathcal{N}(0,\I_p\frac{\epsilon_{mod}^2\cdot \log(1/\delta)}{\epsilon})$ is $(\epsilon,\delta)$-{\em locally differentially private} at the data set $D$. Furthermore, using the fact that $\rL(\theta)$ is $O(GB)$ Lipschitz and the results of \citet[Theorem 1]{jain2014near} and \citet[Theorem 3]{CMS11}, we have:

\begin{enumerate}\itemsep -3pt
	\item \textbf{Improper learning bound:} $\E[|\rL(\hat\theta_T^{(j^*)})-\rL(\theta_T^{(j^*)})|]=G\sigma$.
	\item \textbf{Proper learning bound:} $\E[|\rL(\proj_\C(\hat\theta_T^{(j^*)}))-\rL(\theta_T^{(j^*)})|]=GB\sqrt{p}\sigma $.
\end{enumerate}

Notice that the improper learning bound is independent of any \emph{explicit} dependence on $p$. Also if $T=n^2$ and if we consider $\Lambda$ and all parameters except $n$ to be constants, then the above bounds indicate that the excess error due to the Gaussian noise addition is $O\left(1/n\right)$, which is similar to the excess risk bound of the standard dropout algorithm (Theorem \ref{thm:dropoutGen}). Using arguments similar to \cite{BST} one can show that this bound is tight.


\subsection{Converting dropout optimization to a differentially private algorithm}\label{sec:convDiffPrivate}
Notice that all of our stability guarantees in Section \ref{sec:stabDropout} are independent of the data distribution $\tau(D)$. This has direct implications for differential privacy. We show that using the \emph{propose-test-release} (PTR) framework \cite{DL09,STLasso13}, the dropout heuristic provides a \emph{differentially private} algorithm.

The $(\epsilon,\delta)$-local differential privacy guarantee from Sections \ref{sec:linLoss} and from the previous section can be interpreted as follows. For a given algorithm $\A$ and a level of accuracy $\sigma$ (or a bound on the extra randomness for local differential privacy), there exists a deterministic function $g_{\A,\sigma}:\D^n\to\re$. If $g_{\A,\sigma}(D)$ exceeds a threshold $\zeta$, then $\A(D)$ is $(\epsilon,\delta)$-local differentially private stable at $D$. We drop the dependency of $g$ on $\A$ and $\sigma$ in the sequel.

As an example, consider the local differentially private variant of dropout from the previous section. Recall that to achieve local differential privacy, we introduced Gaussian perturbations to the model. If we require that the standard deviation of the noise added is bounded by $\sigma$ (corresponding to a given level of accuracy) and we expect $(\epsilon,\delta)$-differential privacy, then it suffices for $\Lambda$ (defined in Theorem \ref{thm:modelSGD}) to be at least $\max\left\{O\left(\frac{GB}{\epsilon\sigma}\left(\sqrt{\log T}/T+\frac{1}{n}\right)\right),\frac{\Delta_1}{2}\right\}$, and for $g(D)=\Lambda$.

For the settings in this paper, $g$ has bounded sensitivity (i.e., for any pair of neighboring data sets $D$ and $D'$, $|g(D)-g(D')|\leq \eta$). Notice that in our above example, with $g(D)=\Lambda$, the sensitivity of $g$ is $\eta=\frac{B}{n}$ (where $B$ upper bounds the $L_2$ norm of any feature vector $x$). For differential privacy, bounded sensitivity is essential. The following two-stage algorithm ensures $(\epsilon,\delta)$-differential privacy for dropout.
\begin{enumerate}
	\item \textbf{Noisy estimate of $g$:} $\widehat g\leftarrow g(D)+{\sf Lap}(\eta/\epsilon)$, where ${\sf Lap}(\lambda)$ is drawn from $\frac{1}{2\lambda}e^{-|x|/\lambda}$.
	\item \textbf{Test for safety:} If $\widehat g>\zeta+\left(\eta\cdot\log(1/\delta)/\epsilon\right)$, then execute dropout heuristic, otherwise fail.
\end{enumerate}
One can show using tail bounds that if $g(D)>\zeta+\left(\eta\cdot\log(1/\delta)/\epsilon\right)$, then w.p.~at least $1-\delta$, the above algorithm will output the execution of the dropout heuristic exactly.

\mypar{Optimality of error} Regularized private risk minimization is a well-studied area \cite{CM08,DL09,CMS11,JKT12,KST12,DuchiJW13,song2013stochastic,jain2014near,BST}. Here we provide yet another algorithm for private risk minimization. While the previous algorithms primarily used variants of $L_2$ regularization, we use the dropout regularization for the privacy guarantee. Conditioned on the PTR test passing, the utility of this algorithm follows from the learning bounds obtained for the local differentially private version of dropout which depends on the noise level $\sigma$. While the bounds are tight in general \cite{BST} and there exists other algorithms achieving the same bounds \cite{BST,jain2014near}, since dropout has been often shown to outperform vanilla regularization schemes, it might be advantageous to use dropout learning even in the context of differential privacy.

\subsection{Gaussian mechanism for differential privacy}
\label{app:background}

\mypar{Gaussian mechanism} Let $f:\D^n\to\re^p$ be a vector valued function. For a given data set $D$, the objective is to output an approximation to $f(D)$ while preserving differential privacy. Let $\eta=\max\limits_{D,D' \text{ s.t. }|D\triangle D'|=2}\ltwo{f(D)-f(D')}$ be the \emph{sensitivity} of the function $f$. Gaussian mechanism refers to the following algorithm: Output $f(D)+\mathcal{N}\left(0,\I_p\frac{4\eta^2\log(1/\delta)}{\epsilon^2}\right)$. One can show that the above algorithm is $(\epsilon,\delta)$-differentially private. (See \citet[Lemma 4]{nikolov2013geometry} for a proof.) Using the proof technique of \citet[Lemma 4]{nikolov2013geometry} one can easily show that the theorem also holds for $(\epsilon,\delta)$-local differential privacy at any given data set $D\in\D^n$, with $\eta=\max\limits_{D',|D\triangle D'|=2}\ltwo{f(D)-f(D')}$.

\section{Implications of stability on generalization performance}
\label{app:stabGenPerf}

In this section we focus on the following question: \emph{What implication does stability have on the generalization performance?} Several existing results indeed show formal connection between stability and the generalization performance (See \cite{shalev2010learnability} for a survey). However, the notion of local differential privacy is significantly stronger than the other LOO stability notions used in the literature. This enables us to prove generalization error risk under {\em adversarial} perturbations to a small number of points.  

To this end, we first define the notion of distance of a dataset to ``instability'': Let $\A$ be $(\epsilon,\delta)$-locally differentially private at dataset $D$, then $\Gamma$ is the distance to instability for $\A(D)$, if $\A(D')$ is not $(\epsilon,\delta)$-locally differentially private, where $D'$ is obtained by changing at most $\Gamma$ points in $D$. A similar notion of instability can be defined for model stability (i.e., the model does not change more than $\epsilon_{\sf mod}$ in the $L_2$-norm).

In the below theorems, we show that if less than $\Gamma$ points of $\D$ are changed adversarially, then the generalization error of $\A$ does not increase significantly. That is, the algorithm can tolerate up to $\Gamma$ adversarial noise.

\begin{thm}
Let $f:\C\to\re$ be a $L$-Lipschitz continuous function. For a given data set $D\in\D^n$, if an algorithm $\A$ (with the range space of $\A$ in $\C$) has $\Gamma$ distance to $\epsilon_{\sf mod}$-model instability, then for any data set $D'$ differing in at most $m<\Gamma$, the following holds.
\begin{equation}
\E\limits_{\A}\left[\left|f(\A(D))-f(\A(D'))\right|\right]\leq Lm\epsilon_{\sf mod}.
\label{eq:modelStabaswee}
\end{equation}
\label{thm:modelStab12}
\end{thm}

The proof Theorem \ref{thm:modelStab12} above follows immediately by triangle inequality.

\begin{thm}
Let $f:\C\to\re^+$ be a bounded function with $\max\limits_{x\in\C} |f(x)|\leq B$. For a given data set $D\in\D^n$, if an algorithm $\A$ (with the range space of $\A$ in $\C$) has $\Gamma$ distance to $(\epsilon,\delta)$-local differential privacy, then for any data set $D'$ differing in at most $m<\Gamma$, the following holds.
\begin{equation}
\E\limits_{\A}\left[f(\A(D'))\right]\leq e^{m\epsilon}\E\limits_{\A}\left[f(\A(D))\right]+mB\delta.
\label{eq:kjh1234}
\end{equation}
\label{thm:measureStab12}
\end{thm}

\begin{proof}
By the definition of \emph{distance of a dataset to instability}, and the composition property \cite{DL09} of differential privacy we know that for any data set $D'$ differing in at most $m<\Gamma$ records from $D$, there exists a set $\mathcal{Q}$ s.t. both over the randomness of $\A(D)$ and over the randomness of $\A(D')$, i) the measure $\mu(\mathcal{Q})\leq m\delta$, and ii) for all $s\in{\sf Range}(\A)\setminus\mathcal{Q}, \left|\frac{\mu_{\A(D)}(s)}{\mu_{\A(D')}(s)}\right|\leq m\epsilon$.

Therefore,
\begin{align*}
\E\limits_{\A}[f(\A(D'))]&=\int\limits_{s\not\in\mathcal{Q}}f(s)\mu_{A(D')}(s)ds+\int\limits_{s\in\mathcal{Q}}f(s)\mu_{A(D')}(s)ds\\
&\leq e^{m\epsilon}\int\limits_{s\not\in\mathcal{Q}}f(s)\mu_{A(D)}(s)ds+B\int\limits_{s\in\mathcal{Q}}\mu_{A(D')}(s)ds\\
&\leq e^{m\epsilon}\E\limits_{\A}[f(\A(D))]+mB\delta.
\end{align*}
The last inequality is true because by assumption $f$ only maps to positive real numbers and the total measure on the set $\mathcal{Q}$ is at most $m\delta$. This completes the proof.
\end{proof}

The above given theorems show that for {\em any} $f$, $D$ and $D'$ provides nearly the same error up to multiplicative $e^{m\epsilon}$ factor and additive $mB\delta$ factor, where $\delta$ is $poly(1/n)$. Hence, using $f$ as the excess population-risk function, the above theorem shows that the excess risk due to adversarial corruptions is bounded by $mB\delta$, where $m$ is the number of points changed, $B$ is the absolution bound on function value and $\delta$ is $O(poly(1/n))$.

We now give an explicit bound on the distance to instability for a simple linear loss problem.
For the ease of exposition, we assume the number of dimensions $p$ to be a constant and ignore ${\sf poly} \log$ terms in $n$, and hide them in the $\tilde O(\cdot)$ notation.

\mypar{Example} Consider a model vector $\theta^*\in\C$ and a distribution $\mathcal{N}(0,\Sigma^{-1})$, where $\Sigma$ is a diagonal covariance matrix with the diagonal being $\langle\lambda_1^2,\cdots,\lambda_p^2\rangle$. And let $D=\{(x_1,y_1),\cdots,(x_n,y_n)\}$ be a data set of $n$ entries, where each $x_i\sim \mathcal{N}(0,\Sigma^{-1})$ and $y_i=\ip{x_i}{\theta^*}$. For dropout regularization, the corresponding optimization problem is given in \eqref{eq:dropout}.

Let $\Lambda_\Gamma=\frac{1}{n}\left[\min\limits_{S\subseteq[n],|S|=\Gamma}\min\limits_{j\in[p]}\sum\limits_{i=1,i\not\in S}^n x_i(j)^2\right]$. Intuitively, $\Lambda_\Gamma$ refers to the strong convexity of the expected objective function in \eqref{eq:dropout}, when at most $\Gamma$ entries are adversarially modified in the data set $D$. By the tail bound for Chi-squared distribution, one can show that $\Lambda_\Gamma=\Omega(1)$ with probability at least $1-1/{\sf poly}(n)$ as long as $\lambda_1,\cdots,\lambda_p$ are $\Omega(1)$ and $\Gamma=\tilde o(n)$. Therefore, for a given value of $\epsilon$ and $\delta$ and any data set $D'$ such that $|D\Delta D'|\leq 2\Gamma$, the standard deviation $\sigma$ of the noise needed to obtain $(\epsilon,\delta)$-local differential privacy is $ \widetilde O\left(\frac{\log(1/\delta)}{\epsilon n}\right)$. This implies that for the above value of $\sigma$, \eqref{eq:kjh1234} in Theorem \ref{thm:measureStab12} is true for all $m=\tilde o(n)$.

Similarly for model stability, one can show for $m=\tilde o(n)$ in the above example.


\section{Missing details from experiment section}
\label{app:experiment}

\subsection{Stability experiment for logistic regression}
\label{app:experimentLogistic}

Here we provide the results of the stability experiment for the following data sets
i) Comp.graphics vs comp.windows in Figure \ref{fig:stabilityApp}(a) $1953$ examples and  $31822$ features), and ii) Rec.sport.baseball vs sci.crypt in Figure \ref{fig:stabilityApp}(b) (with  $1985$ examples and $24977$ features). In all the data sets above, we use half the data for training and the rest for testing. For error, we measure the fraction of misclassification.
\begin{figure}[htb]
\centering
\begin{tabular}[h]{cc}
    \includegraphics[width=0.45\textwidth]{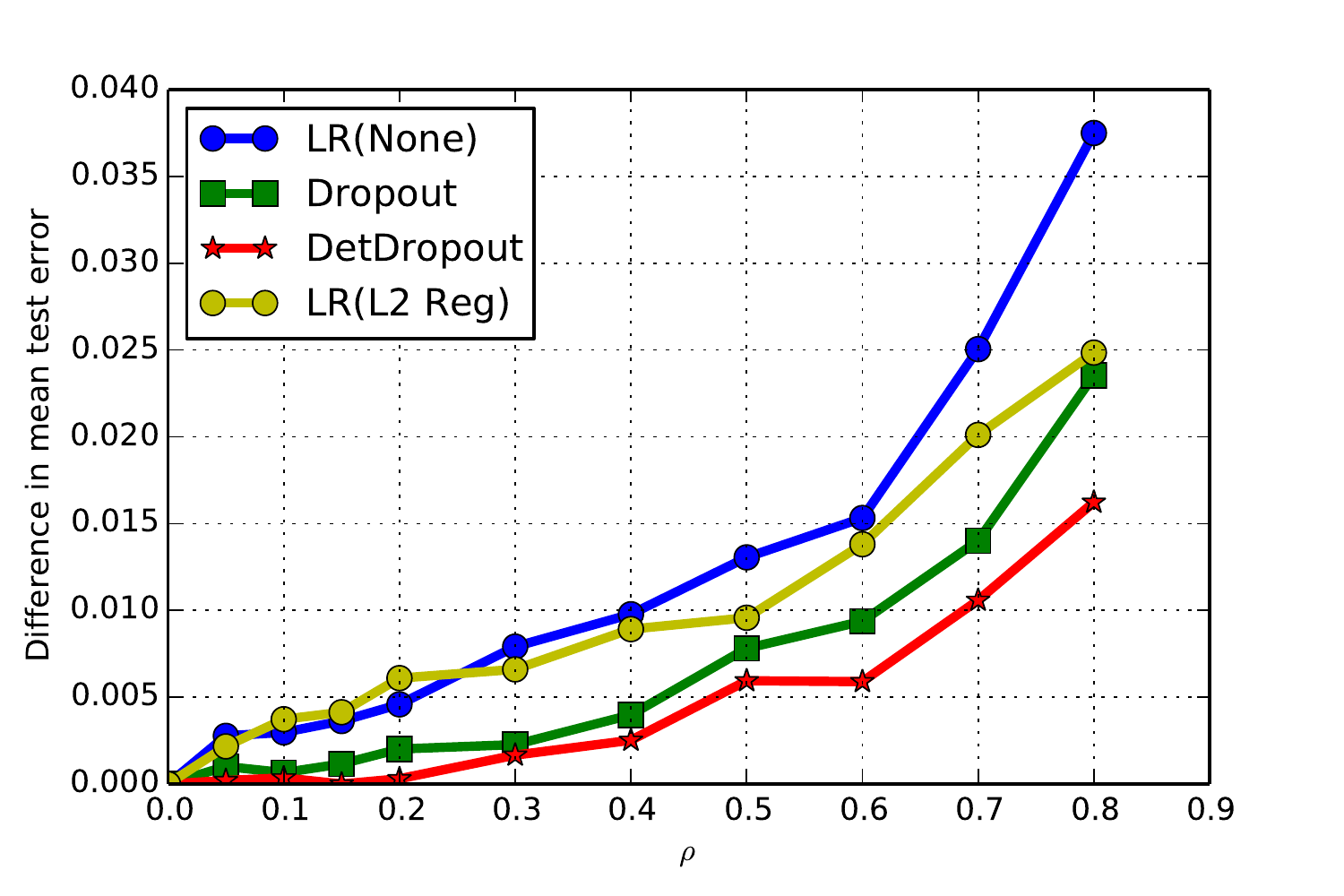}&
    \includegraphics[width=0.45\textwidth]{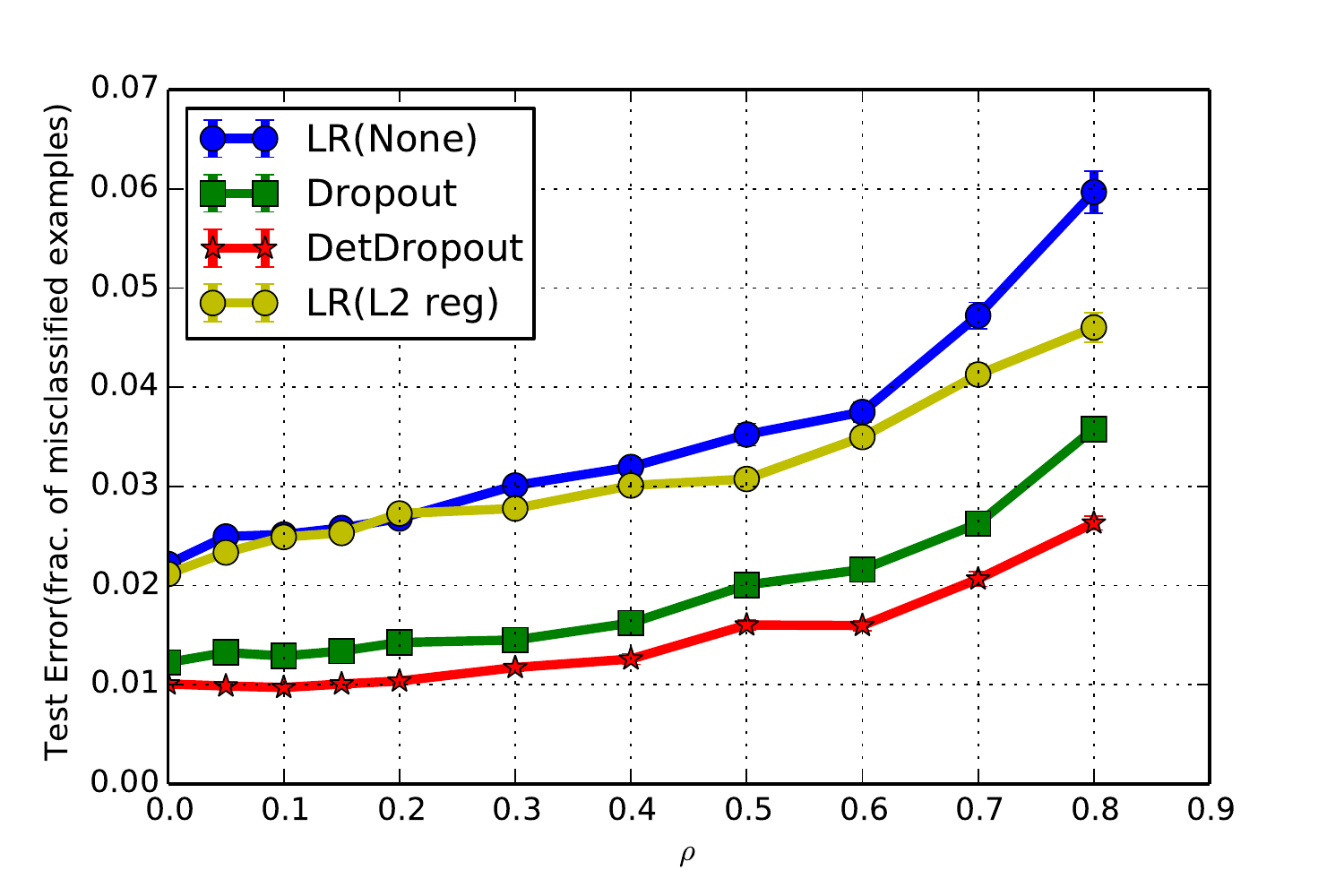}\\
    \multicolumn{2}{c}{(a) Baseball data set}\\
    \includegraphics[width=0.45\textwidth]{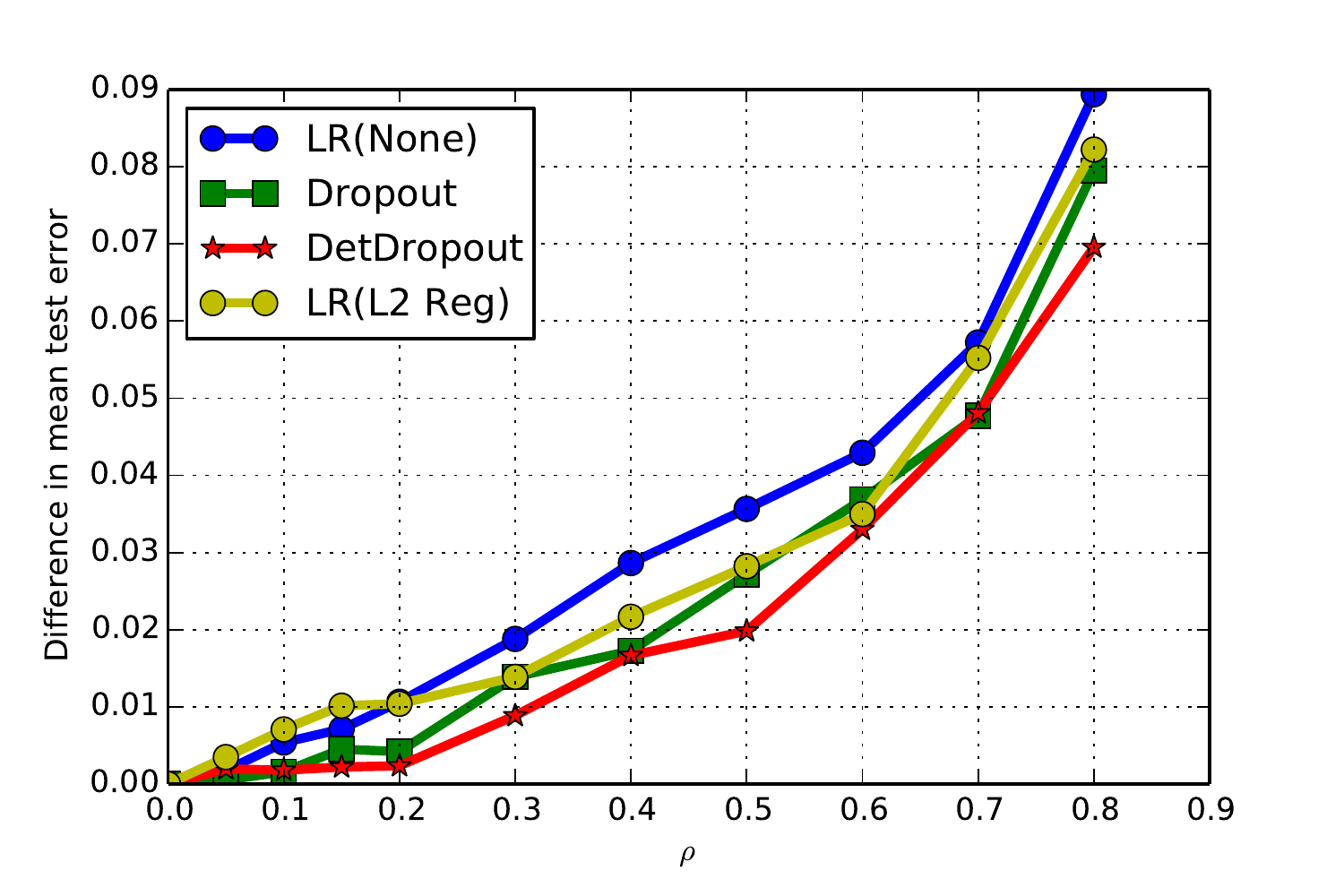}&
    \includegraphics[width=0.45\textwidth]{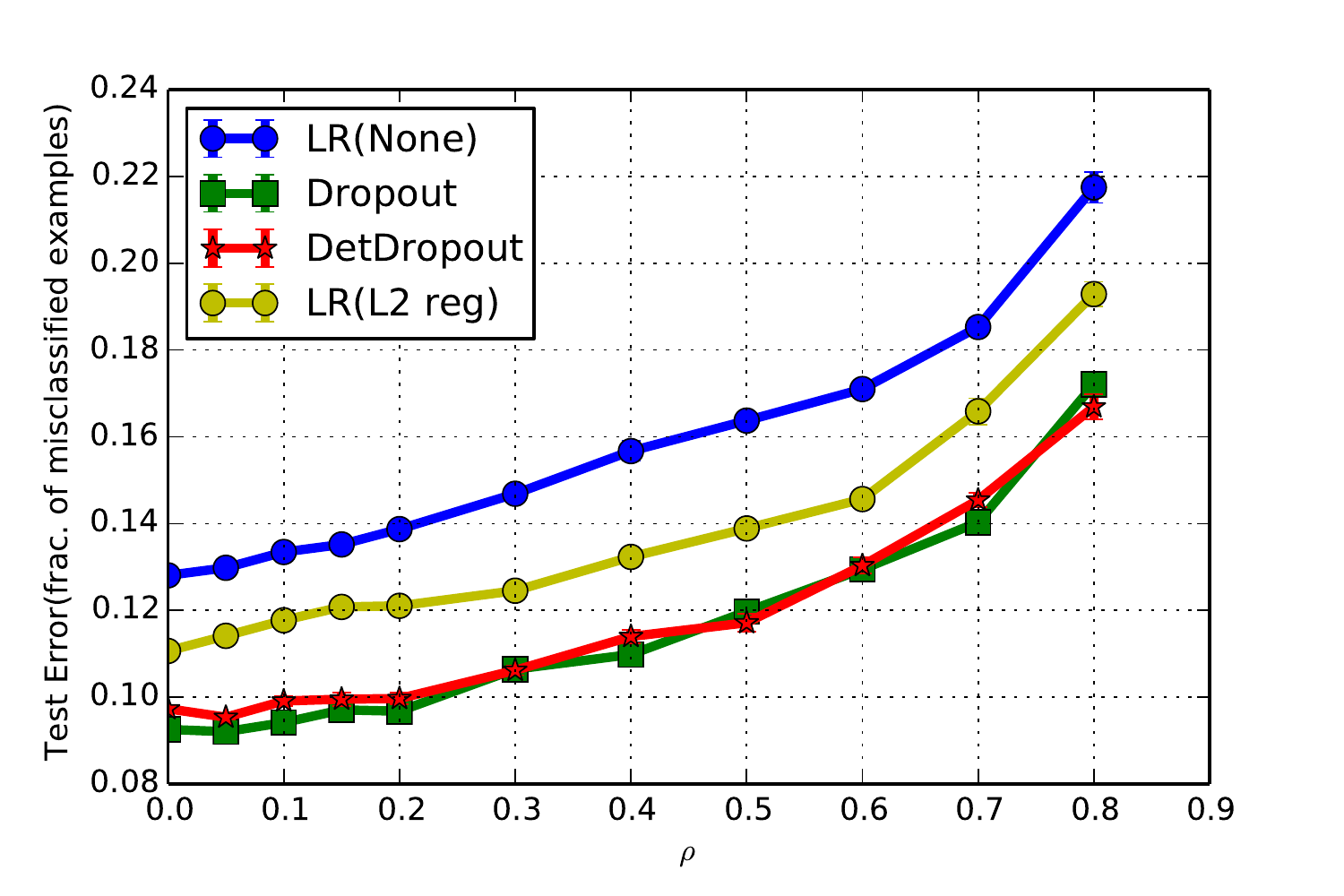}\\
    \multicolumn{2}{c}{(b) Windows data set}
\end{tabular}
\caption{Stability analysis for logistic regression (random sub-sampling)}
\label{fig:stabilityApp}
\end{figure}

\subsection{Stability Experiment for DBNs}
\label{app:experimentDBN}

Here we provide the stability result for the \textbf{Leaves} data set. The data set consists of images of 10 different classes of leaves. The leaves data set consists of 8000 images of leaves. We use 7000 for training and 1000 images for testing (by splitting randomly). For training a DBN on this data set we use a network with a configuration of 784, 1024, 1024, and 10 units respectively. For error, we measure the \# of misclassification.

\mypar{Note} Each DBN was trained for 150 epochs, and experiment was repeated five times.

\begin{figure}[htb!]
\centering
\begin{tabular}[h]{cc}
    \includegraphics[width=0.45\textwidth]{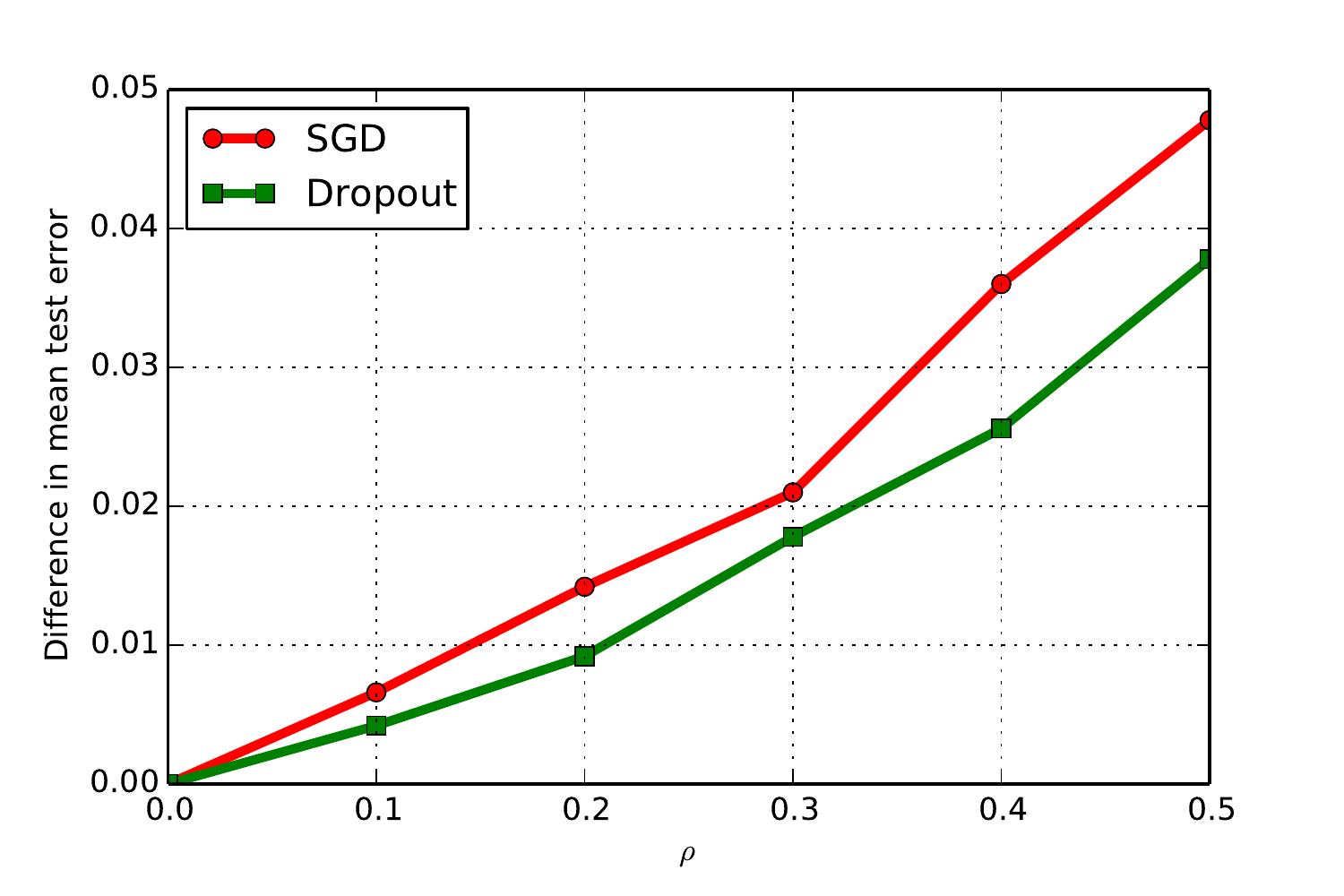}&
    \includegraphics[width=0.45\textwidth]{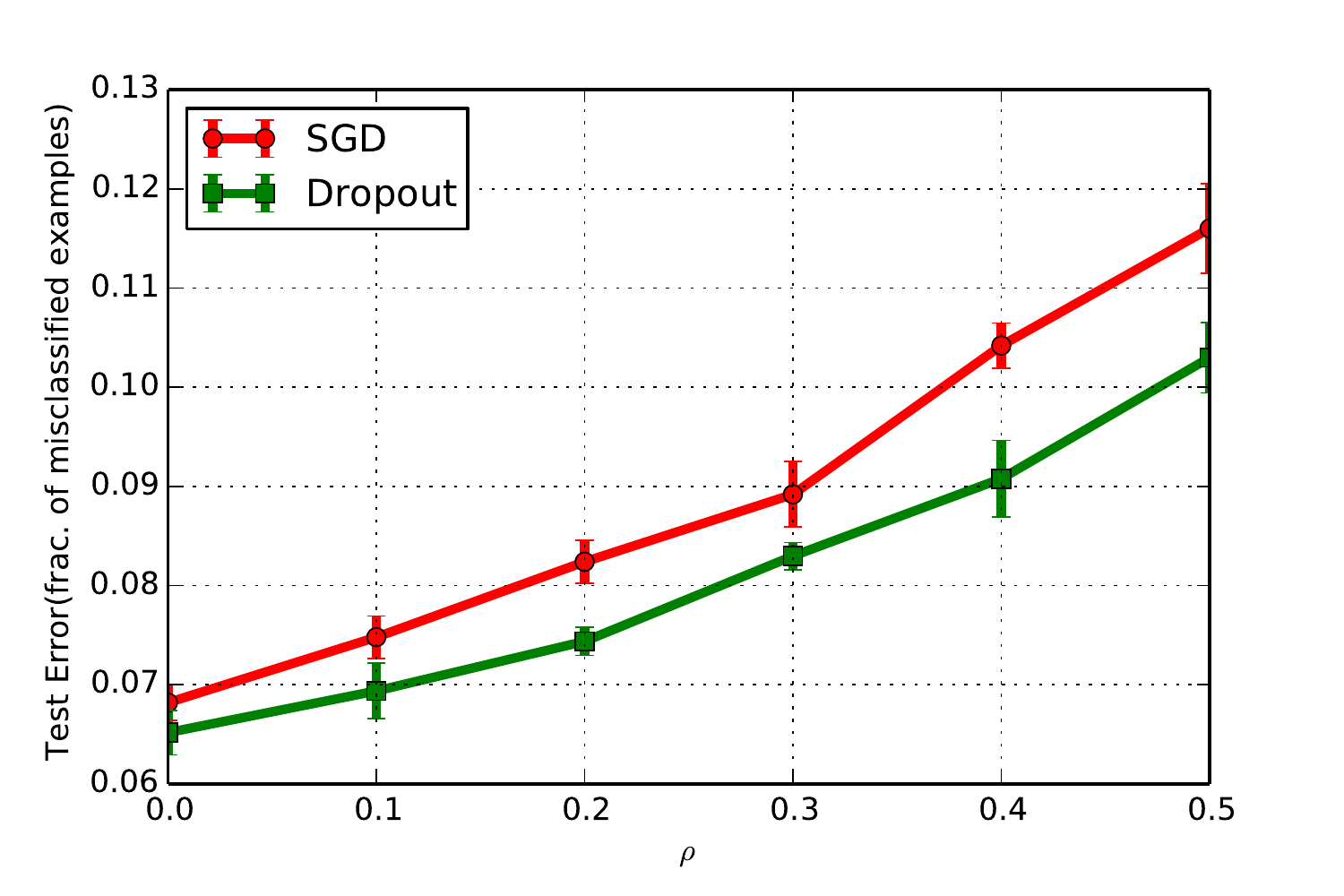}\\
\end{tabular}
\caption{Stability analysis for DBN with Leaves data set (under random sub-sampling).}
\label{fig:dbn_figLeaves}
\end{figure}

\end{document}